\documentclass[11pt]{amsart}

\usepackage[letterpaper,top=3.5cm,bottom=3.5cm,left=3.5cm,right=3.5cm,marginparwidth=1.75cm]{geometry}

\usepackage{appendix}
\usepackage{amssymb}
\usepackage{subfig}
\usepackage{graphicx}
\usepackage{color}
\usepackage{bm}
\usepackage{enumitem}
\usepackage{hyperref}
\usepackage{cleveref}
\newcommand{\N}{\mathcal{N}}

\newcommand{\G}{\mathcal{G}}
\newcommand{\F}{\mathcal{F}}
\newcommand{\E}{\mathbb{E}}

\newcommand{\LL}{\mathrm{L}}

\newcommand{\R}{\mathbb{R}}

\newcommand{\Z}{\mathbb{Z}}

\newcommand{\cS}{\mathcal{S}}

\newcommand{\bigO}{\mathcal{O}}

\DeclareMathOperator*{\argmin}{argmin}

\newcommand{\dd}{\mathrm{d}}

\newtheorem{theorem}{Theorem}
\newtheorem{lemma}[theorem]{Lemma}
\theoremstyle{plain}
\newtheorem{assumption}{Assumption}
\newtheorem{remark}{Remark}

\newenvironment{newremark}[1]{%
    \begin{remark}#1}{%
    \Endofdef\end{remark}%
}
\newcommand{\xqed}[1]{%
    \leavevmode\unskip\penalty9999 \hbox{}\nobreak\hfill
    \quad\hbox{\ensuremath{#1}}}
\newcommand{\Endofdef}{\xqed{\lozenge}}

\definecolor{darkred}{rgb}{.6,0,0}
\definecolor{darkblue}{rgb}{0,0,.7}
\definecolor{darkgreen}{rgb}{0,.7,0}
\definecolor{darkbrown}{rgb}{0.8,0.4,0.4}

\begin{document}

\title[Black-Box Variational Inference]{Adaptive Exponential Integration for Stable Gaussian Mixture Black-Box Variational Inference}

\author{Baojun Che}
\address{School of Mathematical Sciences, Peking University, Beijing, China}
\email{bjche25@stu.pku.edu.cn}

\author{Yifan Chen}
\address{Department of Mathematics, University of California, Los Angeles, CA, USA}
\email{yifanchen@math.ucla.edu}

\author{Daniel Zhengyu Huang}
\address{Beijing International Center for Mathematical Research,  Center for Machine Learning Research, Peking University, Beijing, China}
\email{huangdz@bicmr.pku.edu.cn}

\author{Xinying Mao}
\address{School of Mathematical Sciences, Peking University, Beijing, China}
\email{maoxinying@stu.pku.edu.cn}

\author{Weijie Wang}
\address{School of Mathematical Sciences, Peking University, Beijing, China}
\email{wwj66285509350@stu.pku.edu.cn}

% REQUIRED
\begin{abstract}

Approximating probability distributions known only up to a normalization constant is a central problem in Bayesian inference, particularly in black-box settings arising in scientific computing. Black-box variational inference (BBVI) with Gaussian-mixture families provides a flexible framework for approximating complex posterior distributions, but the associated optimization dynamics are numerically delicate: covariance updates may leave the positive-definite cone, ill-conditioning may force prohibitively small step sizes, and Monte Carlo noise may obscure convergence. In this work, we study the affine-invariant natural-gradient dynamics for Gaussian-mixture variational inference, equivalently the Fisher-Rao gradient flow, from the viewpoint of stable numerical discretization. We propose an exponential integrator for the covariance evolution that preserves positive definiteness unconditionally, together with an adaptive time-stepping strategy that resolves the distinct warm-up and convergence regimes of the dynamics. The resulting scheme also admits a geometric interpretation in terms of Riemannian optimization and mirror descent. For Gaussian targets, we prove exponential convergence of the deterministic iteration and almost-sure convergence under Monte Carlo approximation. For well-separated Gaussian-mixture targets, we establish the corresponding local convergence results. Numerical experiments on multimodal distributions, Neal's multiscale funnel, and a PDE-based Bayesian inverse problem for Darcy flow demonstrate the effectiveness of the proposed method, including examples in dimensions up to \(50\).
\end{abstract}

% REQUIRED

% REQUIRED

\maketitle

\section{Introduction}
Sampling from probability distributions known only up to a normalization constant 
is a central task in Bayesian inference and uncertainty quantification for 
scientific computing. Consider a target distribution with density
\begin{equation}
\pi(\theta) \propto e^{-\Phi_R(\theta)},
\end{equation}
where $\Phi_R(\theta)$ denotes a known potential function. The objective is to 
efficiently generate samples from $\pi$ or to compute expectations with respect 
to it. In large-scale Bayesian inverse problems \cite{kaipio2006statistical,stuart2010inverse,oliver2008inverse,bui2013computational,cui2015data,yan2019adaptive,cao2022bayesian,guo2024ib,henneking2025real}, evaluating $\Phi_R$ typically requires computationally 
expensive forward model simulations, rendering traditional sampling approaches 
such as Markov chain Monte Carlo (MCMC) methods prohibitively costly.

Variational inference \cite{jordan1999introduction,blei2017variational,chen2025accelerating} provides a powerful alternative by 
reformulating sampling as an optimization problem. Rather than drawing samples 
directly from $\pi$, one introduces a tractable parametric family of distributions 
$\{\rho_a\}$, indexed by parameters $a$, and seeks the member that is closest to 
$\pi$ in the sense of the Kullback--Leibler (KL) divergence:
\begin{align}
\label{eq:KL}
\min_{\rho_a} \mathrm{KL}[\rho_a \Vert \pi]
= \min_{\rho_a} \int \rho_a \log \rho_a \,\dd\theta
+ \int \rho_a \Phi_R \,\dd\theta
+ \text{const}.
\end{align}
This optimization-based viewpoint enables efficient approximations in settings where exact sampling is infeasible but accurate approximations are sufficient. Indeed, it underpins many classical approaches in Bayesian inference and inverse problems, including mean-field approaches \cite{jaakkola1998improving, jordan1999introduction, chen2025rotated}, methods based on Gaussian approximations \cite{opper2009variational, huang2022efficient, lambert2022variational}, and Kalman-filter-type methods \cite{oliver2008inverse,iglesias2013ensemble,UKI, chada2020tikhonov,garbuno2020interacting,hanu2023subsampling,bach2024machine,liu2025dropout}.

Single-Gaussian approximations, despite their success, are often inadequate for posteriors that are multimodal, multiscale, or exhibit complex geometric structure. To address these limitations, more expressive variational families have been developed, including Gaussian mixture models~\cite{jaakkola1998improving} and normalizing flows \cite{rezende2015variational,yao2024minimizing}. Among these, Gaussian mixture models provide a particularly useful balance between computational tractability and expressive power~\cite{delon2020wasserstein,huix2024theoretical}.
In many contemporary applications, the potential $\Phi_R$ is accessible only as a black box, with derivatives with respect to $\theta$ either unavailable or unreliable. This setting motivates black-box variational inference (BBVI) methods~\cite{ranganath2014black,han2018stein,domke2023provable}, which avoid direct use of $\nabla \Phi_R$.
Nevertheless, Gaussian mixture BBVI remains numerically challenging. In particular, covariance matrices may lose positive definiteness during optimization, thereby forcing prohibitively small step sizes \cite{lin2019fast,che2025stable}. These difficulties are further amplified by Monte Carlo noise in the gradient estimators, especially in high-dimensional settings, and can lead to slow convergence or even optimization failure \cite{ranganath2014black,paisley2012variational}. This motivates the development of stable, adaptive, and structure-preserving numerical schemes for Gaussian mixture BBVI.

\subsection{Contributions}
This paper develops a stable numerical framework for Gaussian mixture BBVI in the black-box setting by combining natural gradient formulations with adaptive exponential time integration. 
Our main contributions are as follows:
\begin{enumerate}
    \item We propose an adaptive exponential integration scheme for Gaussian mixture BBVI. 
    The covariance update preserves positive definiteness
    unconditionally, while adaptive time-stepping stabilizes the dynamics and separates the warm-up and convergence phases.
    The scheme also admits geometric interpretations in terms of Riemannian optimization and mirror descent.

    \item We provide a convergence analysis of the proposed scheme at the discrete level. For Gaussian targets, we prove deterministic exponential convergence in the noise-free setting and almost-sure convergence under Monte Carlo approximation with a decaying step-size scheduler, thereby clarifying the role of adaptive time stepping in handling poor initialization and stochastic noise. For well-separated Gaussian-mixture targets, we establish the corresponding local convergence results.

    \item We demonstrate the practical effectiveness of the proposed method on a range of challenging numerical examples, including multimodal target distributions, Neal’s multiscale funnel, and a Bayesian inverse problem for Darcy flow.
\end{enumerate}

\subsection{Preliminaries and literature review}

\subsubsection{Natural gradient}

Conventional gradient descent  depends strongly on the chosen parameterization of
the model, which can lead to severe ill-conditioning and the need for very small step sizes, particularly when optimization is performed on curved statistical manifolds. This limitation arises because the standard gradient corresponds to the direction of steepest descent with respect to the Euclidean metric on the parameter space, rather than the intrinsic geometry of the underlying family of probability distributions.

Natural gradient methods~\cite{amari1998natural,martens2020new,wang2022accelerated,carrillo2024fisher} address this issue 
by preconditioning the gradient with the Riemannian metric induced by the statistical 
manifold. This metric is given by the Fisher information matrix (FIM), which 
captures the local curvature of the model:
\begin{align*}
\frac{\mathrm{d} a}{\mathrm{d} t}
= -\mathrm{FIM}(a)^{-1} \nabla_a 
\mathrm{KL}\big[ \rho_a \Vert \pi \big],\  
\mathrm{FIM}(a)
:= \int 
\left( \nabla_a \log \rho_a(\theta) \right)
\left( \nabla_a \log \rho_a(\theta) \right)^{T}
\rho_a(\theta)\,\mathrm{d}\theta.
\end{align*}
Natural gradient flows are invariant under smooth reparameterizations, so the resulting optimization trajectory depends only on the geometry of the 
distribution space. We further 
highlight their affine invariance \cite{goodman2010ensemble,chen2023sampling} with respect to linear transformations of the 
parameter space in the Gaussian and mixture context, which contributes to robustness under anisotropy and ill-conditioning 
(see \Cref{sec:theory}).
Natural gradients have been widely used to accelerate Gaussian variational 
inference~\cite{amari1998natural,opper2009variational, hoffman2013stochastic,garbuno2020affine,chen2023sampling}, since the Fisher 
information matrix admits a closed-form expression for Gaussian families. 
Extending this framework to Gaussian mixture models is challenging, since the exact Fisher information matrix is generally unavailable in closed form and its dimension grows rapidly with the number of mixture components.
Existing methods~\cite{lin2019fast,arenz2020trust,arenz2022unified,chen2024efficient} therefore typically rely on tractable approximations that decouple the updates of the individual components and the mixture weights. Related Wasserstein-gradient-flow approaches~\cite{lambert2022variational} also lead to decoupled componentwise updates.
Nevertheless, numerical stability remains a central bottleneck, since even gradient- or Hessian-based updates may require very small step sizes.

\subsubsection{Black-box variational inference}
\label{sec-review-bbvi}
Many variational inference methods rely on first- or second-order derivative 
information of the potential $\Phi_R$~\cite{bishop2006pattern, blei2017variational}, 
which may be unavailable or unreliable when $\Phi_R$ is accessible only through a 
black-box forward model. Black-box variational inference (BBVI)~\cite{ranganath2014black} 
was introduced to overcome this limitation by constructing unbiased Monte Carlo 
estimates of the gradient of the KL divergence.

Specifically, BBVI employs the score-function (log-derivative) identity to rewrite 
the gradient of the KL divergence with respect to the variational parameters $a$ as
\begin{equation}
\label{eq:bbvi-gradient}
\nabla_a \mathrm{KL}\big[ \rho_a \Vert \pi \big]
= \mathbb{E}_{\rho_a}
\left[
\nabla_a \log \rho_a(\theta)
\left( \log \rho_a(\theta) + \Phi_R(\theta) \right)
\right].
\end{equation}
The expectation in \cref{eq:bbvi-gradient} is approximated by Monte Carlo sampling, yielding an unbiased stochastic gradient estimator that does not require derivatives of $\Phi_R$. However, this approach may suffer from high
variance and step-size sensitivity, motivating extensive work on variance reduction and stabilization techniques
\cite{ranganath2014black,paisley2012variational,Manushi2024framewaork}.

For Gaussian mixture BBVI, an additional challenge is that covariance updates may leave the positive-definite cone, while ill-conditioning may force prohibitively small step sizes. 
Specialized quadrature-based methods, motivated by unscented transform ideas~\cite{julier1995new}, are developed in the authors' previous work \cite{chen2024efficient,che2025stable} to preserve covariance positivity for relatively large time steps.  However, these methods are tailored to Bayesian inverse problems in which $\Phi_R$ has a nonlinear least-squares structure. From a geometric perspective, preservation of covariance positivity is naturally connected to Riemannian optimization on the cone of positive-definite matrices~\cite{pennec2006riemannian,lin2020handling}. In the
present work, we focus on general black-box applications and develop an exponential integrator that rigorously respects the underlying Riemannian
structure and preserves positive definiteness unconditionally. Combined with
adaptive time stepping, this yields a stable and efficient algorithm for Gaussian mixture BBVI.

\subsection{Organization}
\label{ssec:over}
In \cref{sec:main}, we introduce natural gradients for the Gaussian mixture family, and propose an exponential integrator with adaptive time stepping for its discretization, forming the basis of our BBVI method.
\Cref{sec:theory} provides theoretical insights behind the proposed algorithm, including its convergence, connections to manifold optimization, and affine invariance.  
\Cref{sec:implementaion} presents several implementation details, such as the scheduler and annealing-based initialization. Numerical experiments are described in \cref{sec:experiments}, which demonstrate the effectiveness of the proposed framework for Bayesian inference. Finally, concluding remarks are provided in \cref{sec:conclusions}.

\section{Stable Gaussian mixture BBVI}
\label{sec:main}
In this section, we first formulate the natural gradient flow for Gaussian mixture 
variational inference. We then introduce an adaptive 
time integration scheme for its stable discretization, which constitutes our main 
algorithm.

We consider a Gaussian mixture variational family
\[\rho_a^{\rm GM}(\theta) = \sum_{k=1}^{K} w_k \,\mathcal{N}(\theta; m_k, C_k), \quad a := [m_1, \dots, m_K,\; C_1, \dots, C_K,\; w_1, \dots, w_K],\]
where $K$ denotes the number of mixture components. The distribution is 
parameterized by component means $m_k \in \mathbb{R}^{N_\theta}$, covariance 
matrices $C_k \in \mathbb{R}^{N_\theta \times N_\theta}$, and nonnegative weights 
$w_k \in \mathbb{R}_{\ge 0}$ satisfying $\sum_{k=1}^{K} w_k = 1$. These parameters 
are collected into the vector $a$.
\subsection{Natural gradient variational inference}
\label{sec-Natural gradient variational inference}
Natural gradient methods are widely used in variational inference due to their ability 
to accelerate optimization by incorporating local geometric information through 
the Fisher information matrix. However, for Gaussian mixture models, the Fisher 
information matrix does not admit a closed-form expression. Moreover, the full 
matrix has dimension 
$K(N_\theta^2 + N_\theta + 1) \times K(N_\theta^2 + N_\theta + 1)$, making direct 
inversion computationally prohibitive.

In practice, the Fisher information matrix is often approximated by a 
block-diagonal matrix~\cite[Appendix C.8]{chen2024efficient}
\cite{lin2019fast, che2025stable}. This approximation decouples the updates of the individual mixture components and reduces the required linear algebra to matrix operations of size
\(N_\theta \times N_\theta\).  It leads to the following natural gradient flow for each component $k = 1, \dots, K$:
\begin{equation}
\label{eq:gm-ngf}
\begin{split}
        \frac{\dd {m}_{k}}{\dd t} 
        &= -C_k\int \N_k(\theta) \Bigl( \nabla_{\theta} \log\rho_a^{\rm GM}  +  \nabla_{\theta} \Phi_R \Bigr)  \dd\theta\\
        &= -\E_{\N_k}\Bigl[ \theta \bigl(\log\rho_a^{\rm GM} + \Phi_R - \E_{\N_k}[\log\rho_a^{\rm GM} + \Phi_R] \bigr)\Bigr],  
        \\
        \frac{\dd {C}_{k}}{\dd t} 
        &= 
        -C_k\int \N_k(\theta)\bigl(\nabla_{\theta}\nabla_{\theta}\log \rho_a^{\rm GM}  + \nabla_{\theta}\nabla_{\theta}\Phi_R\bigr) \dd\theta C_k\\
        &= 
        -\E_{\N_k}\Bigl[(\theta - m_k)(\theta - m_k)^T \bigl(\log \rho_a^{\rm GM}  + \Phi_R - \E_{\N_k}[\log\rho_a^{\rm GM} + \Phi_R]\bigr)\Bigr],  
        \\
         \frac{\dd \log {w}_{k}}{\dd t}  &= -\int \Bigl(\N_k(\theta) -  \rho_a^{\rm GM}\Bigr)\bigl(\log \rho_a^{\rm GM}  + \Phi_R \bigr) \dd\theta. 
\end{split}
\end{equation}
Here, $\mathcal{N}_k = \mathcal{N}(\theta; m_k, C_k)$ denotes the $k$-th Gaussian 
component, and all quantities depend implicitly on time $t$. Expectations are 
taken with respect to $\mathcal{N}_k$. The logarithmic weight parametrization preserves positivity of the weights, and the constraint \(\sum_{k=1}^K w_k=1\) is enforced throughout the evolution.

The derivative-free expressions in \eqref{eq:gm-ngf} are obtained by integration by parts. Following the “sticking-the-landing” 
principle~\cite{roeder2017sticking}, the resulting expressions ensure that the variance of the stochastic estimators of
the gradients vanishes as $\rho_a^{\rm GM}$ approaches the target density, i.e.,
when $\log \rho_a^{\rm GM} + \Phi_R \approx 0$, enabling low-variance Monte Carlo approximations near convergence.

In the next subsection, we present our main contribution: a stable and practical 
time discretization of the natural gradient flow.

\subsection{Adaptive exponential integration}
\label{ssec:algorithm}
To construct a stable scheme, we reparameterize the covariance matrices 
using a square-root factorization (e.g., the Cholesky decomposition),
$$C_k = L_k L_k^T.$$ 
We further define
\begin{align}
\label{eq:fk-reparameter}
    f_k(t, \theta) = \log \rho_{a(t)}^{\rm GM}\bigl(L_k(t) \theta + m_k(t)\bigr)  + \Phi_R\bigl(L_k(t) \theta + m_k(t)\bigr).
\end{align}
Under this reparameterization, expectations with respect to $\mathcal{N}_k$ can be 
expressed in terms of the standard Gaussian distribution $\mathcal{N} = \mathcal{N}(0,I)$. More precisely, substituting \eqref{eq:fk-reparameter} into \cref{eq:gm-ngf} yields
\begin{subequations}
\label{eq:gm-ngf-repar}
\begin{align}
        \frac{\dd {m}_{k}}{\dd t} 
        &= -L_k \E_{\N}\Bigl[\theta \bigl(f_k - \E_{\N}[f_k] \bigr)\Bigr],\label{eq:gm-ngf-repar-m}
        \\
        \frac{\dd {C}_{k}}{\dd t} 
        &= 
        -L_k\E_{\N}\Bigl[\theta \theta^T\bigl(f_k - \E_{\N}[f_k] \bigr)\Bigr]L_k^T, \label{eq:gm-ngf-repar-C}
        \\
        \frac{\dd \log {w}_{k}}{\dd t} &= -\E_{\N}[f_k] + \sum_{i=1}^{K} w_i \E_{\N}[f_i].
\end{align}
\end{subequations}

A key difficulty in time integration is that, for large step sizes, the covariance 
update may lose positive definiteness~\cite[Section~3]{che2025stable}. To address 
this issue, we define
\begin{equation}
\label{eq:Ek}
    E_k(t) = \E_{\N}\Bigl[\theta \theta^T\bigl(f_k(t,\theta) - \E_{\N}[f_k(t,\theta)]\bigr)\Bigr]
\end{equation} 
and employ an exponential-type integrator for the covariance update \cref{eq:gm-ngf-repar-C}:
\begin{equation}
\label{eq:gm-ngf-update_C}
\begin{split}
        C_{k}(t_n + \Delta t_n) 
        &= 
        L_k(t_n) e^{- E_k(t_n)\Delta t_n}L_k(t_n)^T,
\end{split}
\end{equation}
This update is independent of the particular square-root factorization. Indeed, 
for any alternative factorization $L_k'(t_n) = L_k(t_n) Q$ with orthogonal $Q$, the 
corresponding matrix $E_k'(t_n) = Q^T E_k(t_n) Q$ yields the same covariance update \cref{eq:gm-ngf-update_C}. The update is compatible with the Riemannian structure of the cone of positive definite matrices; see \cref{ssec:manifold_optimization}.
In particular, it preserves positive definiteness of the covariance matrix for
any step size \(\Delta t_n>0\). In the moderate-dimensional regimes considered here,  the additional cost of computing the matrix exponentials remains moderate.

The mean and weight equations are discretized explicitly:
\begin{equation}
\label{eq:gm-ngf-update_mw}
\begin{split}
    m_k(t_n+\Delta t_n) &= m_k(t_n) -\Delta t_n
     L_k(t_n) \E_{\N}\Bigl[\theta \bigl(f_k(t_n,\theta) - \E_{\N}[f_k(t_n,\theta)]\bigr)\Bigr],
     \\
    \log \widehat{w}_k(t_n+\Delta t_n) &= \log w_k(t_n) - \Delta t_n \Bigl(\E_{\N}[f_k(t_n,\theta)] - \sum_{i=1}^{K} w_i \E_{\N}[f_i(t_n,\theta)]\Bigr),
    \\
    w_k(t_n+\Delta t_n) &= \frac{\widehat{w}_k(t_n+\Delta t_n)}{\sum_{i=1}^{K}\widehat{w}_i(t_n+\Delta t_n)}.
\end{split}
\end{equation}
The mean variable is unconstrained and is therefore updated by forward Euler. The logarithmic weight update, followed by normalization, yields a multiplicative weight update. Equivalently, it is an exponentiated-gradient~\cite{kivinen1997exponentiated}, or entropy mirror-descent~\cite{beck2003mirror}, step on the probability simplex; see \cref{ssec:manifold_optimization}.

All expectations in \cref{eq:gm-ngf-repar,eq:Ek} are approximated using Monte Carlo sampling, as is standard in BBVI. For each component $k$, we draw $J$ samples 
$\{\theta_k^j\}_{j=1}^J \sim \mathcal{N}(0,I)$ and use them to estimate 
$\mathbb{E}_{\mathcal{N}}[f_k]$ and the corresponding centered moments $\mathbb{E}_{\mathcal{N}}[\theta f_k]$ and $\mathbb{E}_{\mathcal{N}}[\theta\theta^Tf_k]$. This 
procedure requires $JK$ evaluations of $\Phi_R$ per time step.

Finally, the time step $\Delta t_n$ is selected adaptively:
\begin{equation}
\label{eq:adaptive-dt}
\Delta t_n
= \min\left\{
\Delta t_{\max}\,\eta(t_n),
\frac{\beta}{\max_k \|E_k(t_n)\|_2}
\right\},
\end{equation}
where $\Delta t_{\max}$ is a prescribed maximum step size, $\beta$ is a stability 
parameter, and $\eta(t)$ is a scheduling function that decays from $1$ to a small 
value to mitigate Monte Carlo noise. The matrices \(E_k(t_n)\) are estimated as part of the covariance update in \cref{eq:Ek}.
The two terms in \cref{eq:adaptive-dt} are typically active in different phases of the dynamics: the second term controls the step size during the initial warm-up phase, when \(\|E_k(t_n)\|_2\) may be large, while the first term, \(\Delta t_{\max}\eta(t_n)\), usually becomes dominant near convergence.
We also empirically examine whether an additional weight-based restriction is beneficial
for the exponential weight update; the results indicate that such a restriction is
not necessary, see \cref{ssec:discuss-dt}.
A detailed discussion of this adaptive 
time-stepping strategy is provided in \cref{ssec:convergence-study}.

\section{Theoretical insights}
\label{sec:theory}
In this section, we study the aforementioned variational inference algorithm from two perspectives. 
In \cref{ssec:convergence-study}, we analyze its convergence at the discrete level, showing that adaptive time-stepping is required to guide the transition from the warm-up phase to the convergence phase, and that the scheduler plays an important role in controlling the noise.
In \cref{ssec:manifold_optimization}, we present a complementary manifold optimization perspective, which connects the proposed algorithm with mirror descent and helps explain its covariance properties, weight positivity, and affine invariance.

\subsection{Convergence study}
\label{ssec:convergence-study}
In this subsection, we study the convergence properties of the proposed algorithm at the discrete level for the Gaussian and well-separated Gaussian mixture cases.

\subsubsection{Gaussian case}
\label{sssec:Gaussian-convergence-study}
We first study the convergence properties of the proposed algorithm at the discrete level for the Gaussian case, i.e., 
$\Phi_R(\theta)=\frac{1}{2}(\theta - m_{\star})^TC_{\star}^{-1}(\theta - m_{\star})$, with mode number $K=1$. 
Let \begin{equation}
\label{eq:Sigma_v}
   \Sigma_n=C_{\star}^{-\frac{1}{2}}C(t_{n})C_{\star}^{-\frac{1}{2}}\quad \textrm{and} \quad v_n=C_{\star}^{-\frac{1}{2}}\bigl(m(t_{n}) - m_{\star}\bigr), 
\end{equation} and assume the integral terms in~\cref{eq:gm-ngf-repar} are computed exactly (whereas in practice these terms are usually computed by Monte Carlo sampling, which introduces noise). The updates of the mean and covariance in \cref{eq:gm-ngf-update_C,eq:gm-ngf-update_mw} then take the form
\begin{align}
    \label{eq:Gaussian-matrix}
    v_{n+1}  = (I-\Delta t_n \Sigma_n)v_n, \quad \Sigma_{n+1}  = \LL_n e^{\Delta t_n(-\LL_n^T \LL_n+I)} \LL_n^T =  h_n(\Sigma_n),
\end{align}
where 
\begin{align}
\label{eq:LL-h}
\LL_n = C_{\star}^{-\frac{1}{2}} L(t_n) \quad \textrm{and} \quad h_n(x)=xe^{\Delta t_n(1-x)}.  
\end{align}
The derivation can be found in \cref{Proof:convergence}. For the noise-free case, we fix the scheduler $\eta(t)=1$ and do not need to decrease it. The time-step size $\Delta t_n$ defined in \cref{eq:adaptive-dt} then becomes
\begin{align}
\label{eq:adaptive-dt-gaussian}
    \Delta t_n = \min\bigl\{\Delta t_{\max}, \frac{\beta}{\|\Sigma_n-I\|_2}\bigr\}.
\end{align}
Here we used the fact $\lVert E_k(t_n)\Vert_2 = \lVert I - \LL_n^T \LL_n\Vert_2 = \lVert \Sigma_n-I\Vert_2$.
With this adaptive time-stepping, we show in the following theorem that the iteration \cref{eq:Gaussian-matrix} converges exponentially fast (i.e., $\Sigma_n \rightarrow I$ and  $v_n \rightarrow 0$) and exhibits only logarithmic dependence on the norm of its initial condition $\Sigma_0$ and $v_0$. This result is consistent with the continuous-level convergence analysis in \cite[Theorem 5.6]{chen2023sampling} and related works~\cite{garbuno2020interacting,carrillo2021wasserstein,burger2023covariance,chen2023gradient}, where the error decays as $\bigO(e^{-t})$, with hidden constants depending on the initial condition and the solution. The proof is deferred to \Cref{proof_of_theorem-Sigma_n-no-noise}.

\begin{theorem}
\label{theorem-Sigma_n-no-noise}
 Assume that $\Delta t_{\rm max} \leq 1$, $\beta \leq 1$, and that the initial covariance matrix $\Sigma_0$ is symmetric positive definite with minimum and maximum eigenvalues $\lambda_{\rm min}(\Sigma_0)$ and $\lambda_{\rm max}(\Sigma_0)$. Given an error tolerance $\epsilon < 1$, for iterations \cref{eq:Gaussian-matrix} with the adaptive time-stepping 
\cref{eq:adaptive-dt-gaussian}, there exists $N = \mathcal{O}\bigl(|\log \lambda_{\rm min}(\Sigma_0)| + |\log \lambda_{\rm max}(\Sigma_0)| + \max\{\log\|v_0\|_2, 0\} + \log \frac{1}{\epsilon} \bigr)$, such that  
for all $n > N$, 
\begin{equation}
\label{eq:theorem-Sigma_n-no-noise}
    \|\Sigma_n-I\|_2 \leq \epsilon,\quad \|v_n\|_2 \leq \epsilon,
\end{equation} where the hidden constant in the $\mathcal{O}(\cdot)$ depends only on $\Delta t_{\max}$ and $\beta$.   
\end{theorem}

\begin{newremark}
   The convergence behavior proceeds in two phases. In the warm-up phase, all eigenvalues of $\Sigma_n$ approach I within a neighborhood of size independent of $\Sigma_0$ after $\mathcal{O}\bigl(|\log \lambda_{\min}(\Sigma_0)| + |\log \lambda_{\max}(\Sigma_0)|\bigr)$ iterations. The adaptive time step $\frac{\beta}{\lVert \Sigma_n - I\rVert_2}$ ensures exponential convergence to this neighborhood while preventing numerical instability. This is followed by the convergence phase, in which $\Sigma_n$ converges to $I$ with error at most $\epsilon$ in $\mathcal{O}\bigl(\log \frac{1}{\epsilon} \bigr)$ iterations. Meanwhile, the iteration for $v_n$ yields a contraction mapping due to the adaptive time-stepping in \cref{eq:adaptive-dt-gaussian}, which guarantees a contraction factor strictly less than one. Consequently, $v_n$ converges to $0$ within $\mathcal{O}\bigl(\log\|v_0\|_2 + \log \frac{1}{\epsilon} \bigr)$ iterations.
\end{newremark}

\begin{newremark}
   When only a single term in \cref{eq:adaptive-dt-gaussian} is used, numerical instability or loss of exponential convergence can occur. For example, if only $\Delta t_{\rm max}$ is used,  then for a large initial condition (in the warm-up phase), the first iteration yields $$\Sigma_1 = \Sigma_0 e^{-\Delta t_{\rm max} (\Sigma_0 - 1)},$$ which remains positive but becomes exponentially small. Conversely, if only $\frac{\beta}{\|\Sigma_n-I\|_2}$ is used, then $\Sigma_n$ oscillates around $I$ in the convergence phase. Specifically, when $\Sigma_n \in (e^{-\beta/2}, e^{\beta/2})$, one finds $\Sigma_{n+1} \notin (e^{-\beta/2}, e^{\beta/2})$.
\end{newremark}

Next, we consider the effect of the scheduler in the presence of noise. When the integral terms in~\cref{eq:Ek,eq:gm-ngf-update_mw} are estimated via the Monte Carlo method,  stochastic errors are introduced. 
We write the estimated integrals under errors as 
$$\E_\N[\theta\theta^T (f-\E_\N[f])] = \LL_n^T \LL_n-I-\Omega_n\quad \textrm{and}  \quad\E_\N[\theta (f-\E_\N[f])] = \LL_n^T v_n + u_n ,$$ where $\LL_n$ defined in \cref{eq:LL-h} is a square root of $\Sigma_n$ satisfying $\Sigma_n=\LL_n\LL_n^T$.
The terms ${\Omega}_n$ and $u_n$ denote the stochastic errors in the covariance and mean updates at the $n$-th iteration arising from the Monte Carlo approximations.
Accordingly, the updates of the mean and covariance, in place of \cref{eq:Gaussian-matrix}, are given by 
\begin{align}
\label{eq:stoch_mean_covariance_update_analysis}
    v_{n+1} = v_n - \Delta t_n \left( \LL_n(\LL_n^Tv_n + u_n) \right)\quad \textrm{and}\quad % \bigl(I - \Delta t_n \LL_n(\LL_n^T + \widetilde{\Omega}_n)\bigr)v_n
    \Sigma_{n+1}= \LL_n e^{\Delta t_n(-\LL_n^T \LL_n+I+\Omega_n)} \LL_n^T.
\end{align}
 In the stochastic setting, we primarily demonstrate that when Gaussian variational inference has reached a certain accuracy level, an appropriately decaying scheduler guarantees convergence. 
\begin{assumption}
\label{assum:stoch_convergence}
For all $n\geq 1$, we make the following assumptions:
\begin{enumerate}
    \item $\lambda_{\rm max}(\Sigma_n), \|v_n\|_2<C_0$, with some $C_0>1.$
    \item The zero mean Monte Carlo approximation noise for the covariance and mean updates are $\Omega_n$ and $u_n$, where $\Omega_n$ is a symmetric noise matrix.
\end{enumerate}
\end{assumption}

The time step in \cref{eq:adaptive-dt} is given by $\Delta t_n= \min\{\eta(t_n) \Delta t_{\max},\frac{\beta}{\|-\LL_n^T \LL_n+I+\Omega_n\|_2}\}\ $, which depends on both $\Sigma_n$ and $\Omega_n$. Under the above assumptions, we have the following convergence guarantee. The proof, deferred to
\cref{proof_of_thm:Sigma_n-noise}, uses a Lyapunov function, summability
estimates for the stochastic errors, and the Robbins-Siegmund theorem~\cite{robbins1971a}.

\begin{theorem}
    Under \cref{assum:stoch_convergence}, the update scheme given by \cref{eq:stoch_mean_covariance_update_analysis} with the time step $\Delta t_n= \min\{\eta(t_n) \Delta t_{\max},\frac{\beta}{\|-\LL_n^T \LL_n +I+\Omega_n\|_2}\} $,  where $\sum_n\eta(t_n)=+\infty$ and $\sum_n\eta(t_n)^2
    <+\infty$,  ensures the almost sure convergence of $\Sigma_n$ to the identity matrix $I$ and $v_n$ to $0$.

    \label{thm:Sigma_n-noise}
\end{theorem}

\subsubsection{Well-separated Gaussian mixture case}
\label{sssec:Gaussian-mixture-convergence-study}
We next study the discrete-time convergence of the proposed algorithm in the
well-separated Gaussian-mixture regime, i.e., 
\begin{equation}
    \pi(\theta)=\sum_{k = 1}^{K} w_{\star,k} \mathcal{N}(\theta; m_{\star,k}, C_{\star,k}),
\end{equation}
and assume that the components are sufficiently far apart so that the interactions between distinct components can be neglected. We consider BBVI with the same number of mixture components. 
Since the underlying optimization problem is highly nonconvex, the flow may in general converge to a local minimizer and fail to capture some modes. Nevertheless, its evolution typically consists of two stages: a warm-up phase and a convergence phase. During the warm-up phase, the Gaussian components explore the target distribution, and the entropy term in the KL divergence acts as a repulsive force that drives the components apart and promotes exploration; the annealing-based initialization described in \cref{ssec:annealing} is particularly helpful in this stage. During the convergence phase, the Gaussian components settle near the appropriate modes. In the present analysis, we assume that each Gaussian component has already been matched to its corresponding target mode, and we focus exclusively on this convergence regime. More precisely, we assume that the modes are well-separated, so that the overlap between different components is negligible.

Under this well-separation assumption, when $\theta\sim\N(0,I)$, $f_k$ in \cref{eq:fk-reparameter} can be approximated by
\begin{equation}
 \label{eq:fk-reparameter-approx}
    \begin{split}
    f_k(t, \theta) 
    &= \log \sum_{i=1}^{K} w_i\mathcal{N}_i\bigl(L_k(t) \theta + m_k(t)\bigr)  -  \log \sum_{i=1}^{K} w_{\star,i}\mathcal{N}\bigl(L_k(t) \theta + m_k(t); m_{\star,i}, C_{\star,i}\bigr)\\
    &\approx \log \Bigl[w_k\mathcal{N}_k\bigl(L_k(t) \theta + m_k(t)\bigr)\Bigr]  -    \log \Bigl[w_{\star,k}\mathcal{N}\bigl(L_k(t) \theta + m_k(t); m_{\star,k}, C_{\star,k}\bigr)\Bigr]\\
    &= -\frac{1}{2}\theta^T\theta + \frac{1}{2}(L_k\theta + m_k - m_{\star,k})^TC_{\star,k}^{-1}(L_k\theta + m_k - m_{\star,k})  - \frac{1}{2}\log\frac{|C_k| }{|C_{\star,k}|}   + \log\frac{w_k}{w_{\star,k}}.  
    \end{split}
\end{equation}
That is,  $f_k$ depends only on $L_k, m_k, w_k$ and the corresponding target parameters $C_{\star,k}$, $m_{\star,k}$, $w_{\star,k}$, with the influence of the other modes neglected. We introduce the normalized covariance, mean, and weight
\begin{equation}
   \Sigma_{n,k}=C_{\star,k}^{-\frac{1}{2}}C_{k}(t_{n})C_{\star,k}^{-\frac{1}{2}}, \quad 
   v_{n,k}=C_{\star,k}^{-\frac{1}{2}}\bigl(m_{k}(t_{n}) - m_{\star,k}\bigr)
   , \quad r_{n,k} = \frac{w_{k}(t_n)}{w_{\star,k}},
\end{equation}
and define the square-root matrix of the normalized covariance by
\begin{equation}
\LL_{n,k} = C_{\star,k}^{-\frac{1}{2}} L_{k}(t_n) \quad \textrm{such that}\quad  \LL_{n,k} \LL_{n,k}^T = \Sigma_{n,k}.
\end{equation}
Assuming that the integral terms in~\cref{eq:gm-ngf-repar} are computed exactly, the updates for the mean, covariance, and weights in \cref{eq:gm-ngf-update_C,eq:gm-ngf-update_mw} take the form
\begin{subequations}\label{eq:Gaussian-mixture-matrix}
    \begin{align}
    v_{n+1,k}  &= (I-\Delta t_n \Sigma_{n,k})v_{n,k}, 
    \\ 
    \Sigma_{n+1,k}  &= \LL_{n,k} e^{\Delta t_n(-\LL_{n,k}^T \LL_{n,k}+I)} \LL_{n,k}^T,
    \\
    \widehat{w}_{n+1,k}  &= w_{n,k} \bigl(\frac{w_{\star,k}}{w_{n,k}}\bigr)^{\Delta t_n}e^{-\frac{\Delta t_n}{2} \bigl(\textrm{tr}[\Sigma_{n,k}-I] +  v_{n,k}^Tv_{n,k} -  \log|\Sigma_{n,k}|\bigr)} ,
    \\
    w_{n+1,k}  &= \frac{\widehat{w}_{n+1,k}}{\sum_{i=1}^{K}\widehat{w}_{n+1,i}}.
    \end{align}
\end{subequations}
The derivation can be found in \cref{Proof:convergence}. For the noise-free case, we fix the scheduler $\eta(t)=1$ and do not need to decrease it. The time-step size $\Delta t_n$ defined in \cref{eq:adaptive-dt} then becomes
\begin{align}
\label{eq:adaptive-dt-gaussian-mixture}
    \Delta t_n = \min\bigl\{\Delta t_{\max}, \frac{\beta}{\max_k\{\|\Sigma_{n,k}-I\|_2 \}}\bigr\}.
\end{align}
Here we used the fact $\lVert E_k(t_n)\Vert_2 = \lVert I - \LL_{n,k}^T \LL_{n,k}\Vert_2 = \lVert \Sigma_{n,k}-I\Vert_2$.
With this adaptive time-stepping, we show in the following theorem that the iteration \cref{eq:Gaussian-mixture-matrix} also converges exponentially fast. The proof is deferred to \cref{proof_of_theorem-Gaussian-Mixture-no-noise}.
 
\begin{theorem}
\label{theorem-Gaussian-Mixture-no-noise}
Under the well-separated approximation \eqref{eq:fk-reparameter-approx}, assume that $\Delta t_{\rm max} \leq 1$, $\beta \leq 1$, and that the initial covariance matrices $\Sigma_{0,k}$ are symmetric positive definite for all $1 \le k \le K$. 
Let
\[
\underline\lambda_0=\min_{1\le k\le K}\lambda_{\min}(\Sigma_{0,k}),
\quad
\overline\lambda_0=\max_{1\le k\le K}\lambda_{\max}(\Sigma_{0,k}),
\quad
V_0=\max_{1\le k\le K}\|v_{0,k}\|_2, \quad r_{n,k}=\frac{w_{n,k}}{w_{\star,k}}.
\]
Given an error tolerance $\epsilon < \min \left\{ \frac{1}{2}, \frac{\beta}{\Delta t_{\rm max}} \right\}$, for iterations \cref{eq:Gaussian-mixture-matrix} with the adaptive time-stepping \cref{eq:adaptive-dt-gaussian-mixture}, there exists $N = 
\mathcal{O}\bigl(|\log \underline\lambda_0| + |\log \overline\lambda_0| + \max\{\log V_0, \max_{1 \le k \le K} \left| \log r_{0,k} \right|, 0\}  + \log \frac{1}{\epsilon}  + \log N_\theta \bigr)$, such that  
for all $n > N, 1 \leq k \leq K$, 
\begin{equation}
% \label{eq:theorem-Sigma_n-no-noise-gm}
    \|\Sigma_{n,k}-I\|_2 \leq \epsilon,\quad \|v_{n,k}\|_2 \leq \epsilon,\quad |\log r_{n,k}| \leq \epsilon,
\end{equation} where the hidden constant in the $\mathcal{O}(\cdot)$ depends only on  $\Delta t_{\max}$ and $\beta$. 

\end{theorem}

Next, we consider the effect of the scheduler in the presence of noise. When the integral terms in~\cref{eq:Ek,eq:gm-ngf-update_mw} are estimated via the Monte Carlo method,  stochastic errors are introduced. 
We write the estimated integrals under errors as 
\begin{equation}
   \begin{split}
\E_\N[\theta\theta^T (f_k-\E_\N[f_k])] &= \LL_{n,k}^T \LL_{n,k}-I-\Omega_{n,k}, \\
    \E_\N[\theta (f_k-\E_\N[f_k])] &= \LL_{n,k}^T v_{n,k} + u_{n,k}, \\
    \E_\N[f_k] &= \frac{1}{2} \bigl(\textrm{tr}[\Sigma_{n,k}-I] +  v_{n,k}^Tv_{n,k} -  \log|\Sigma_{n,k}|\bigr) + \log \frac{w_{n,k}}{w_{\star,k}} - s_{n,k}.
   \end{split} 
\end{equation}
The terms ${\Omega}_{n,k}$, $u_{n,k}$, and  $s_{n,k}$ denote the stochastic errors in the covariance, mean, and weight updates at the $n$-th iteration arising from the Monte Carlo approximations.
Accordingly, the updates of the mean, covariance, and weight, in place of \cref{eq:Gaussian-mixture-matrix}, are given by 
\begin{subequations}
    \label{eq:Gaussian-mixture-stoch}
    \begin{align}
    v_{n+1,k}  &= v_{n,k} - \Delta t_n \left( \LL_{n,k}(\LL_{n,k}^Tv_{n,k} + u_{n,k}) \right),
    \\ 
    \Sigma_{n+1,k}  &= \LL_{n,k} e^{\Delta t_n(-\LL_{n,k}^T \LL_{n,k}+I + \Omega_{n,k})} \LL_{n,k}^T ,
    \\
    \widehat{w}_{n+1,k}  &= w_{n,k} \bigl(\frac{w_{\star,k}}{w_{n,k}}\bigr)^{\Delta t_n}e^{-\frac{\Delta t_n}{2} \bigl(\textrm{tr}[\Sigma_{n,k}-I] +  v_{n,k}^Tv_{n,k} -  \log|\Sigma_{n,k}|\bigr)} e^{\Delta t_n s_{n,k}},
    \\
    w_{n+1,k}  &= \frac{\widehat{w}_{n+1,k}}{\sum_{i=1}^{K}\widehat{w}_{n+1,i}}.
    \end{align}
\end{subequations}

In the stochastic setting, we show that, once the Gaussian mixture variational
iteration has entered a sufficiently accurate regime, its convergence behavior is
analogous to that of the single Gaussian case. Under a suitable boundedness
assumption and a decaying scheduler, we obtain the following almost-sure
convergence result. The proof is given in \cref{Proof_of_thm:GM-noise}.

\begin{assumption}
\label{assum:GM-stoch_convergence}
For all $n\geq 1, 1\le k \le K$, we make the following assumptions:
\begin{enumerate}
    \item $\lambda_{\max}(\Sigma_{n,k}), \|v_{n,k}\|_2,|\log r_{n,k}|<C_0$, with some $C_0>1.$
    \item The zero mean Monte Carlo approximation noise for the covariance, mean, and weight updates are $\Omega_{n,k}, u_{n,k}$ and $s_{n,k}$, where $\Omega_{n,k}$ is a symmetric noise matrix.
\end{enumerate}
\end{assumption}

\begin{theorem}
    Under the well-separated approximation \eqref{eq:fk-reparameter-approx} and \cref{assum:GM-stoch_convergence}, the update scheme given by \cref{eq:Gaussian-mixture-stoch} with the time step 
    $$\Delta t_n= \min\bigl\{\eta(t_n) \Delta t_{\max},\frac{\beta}{\max_{k}\{\|-\LL_{n,k}^T \LL_{n,k} +I+\Omega_{n,k}\|_2\}}\bigr\},$$  
    where $\sum_n\eta(t_n)=+\infty$ and $\sum_n\eta(t_n)^2
    <+\infty$,  ensures the almost sure convergence of $\Sigma_{n,k}$ to the identity matrix $I$, $v_{n,k}$ to $0$ and $\log r_{n,k}$ to $0$ for every \(1\le k\le K\).

    \label{thm:GM-noise}
\end{theorem}

\begin{newremark}

In the numerical experiments, we observe exponential convergence in the noisy setting even for Gaussian mixture targets whose modes are not well-separated, provided that the number of Gaussian mixture components equals the number of target modes and that the approximation captures all modes; see \cref{sec:exponential}.
\end{newremark}

\subsection{Relation to manifold optimization and mirror descent}
\label{ssec:manifold_optimization}
In this subsection, we study the proposed algorithm from a manifold optimization perspective, showing that it is naturally designed within the framework of mirror descent \cite{beck2003mirror} to ensure positivity and affine invariance.

The covariance update in \cref{eq:gm-ngf-update_C} can be interpreted as a matrix optimization problem constrained to the Riemannian manifold of symmetric positive definite matrices $\cS_{++}$. On this manifold, the tangent space at a point $X\in\cS_{++}$, denoted $T_X \cS_{++}$, consists of symmetric matrices. The affine invariant Riemannian metric and distance are defined as \cite{pennec2006riemannian,han2021riemannian}
\begin{align*}
   g_X(\sigma_1, \sigma_2) = {\rm tr}(X^{-1}\sigma_1 X^{-1}\sigma_2) ,\quad
    {\rm d}(X,Y) = \lVert \log(X^{-\frac{1}{2}}YX^{{-\frac{1}{2}}^T}) \rVert_F, 
\end{align*}
for $X,Y\in \cS_{++}$ and $\sigma_1, \sigma_2\in T_X \cS_{++}$. Here $X^{\frac{1}{2}}$ denotes a square-root matrix satisfying $X = X^{\frac{1}{2}} X^{{\frac{1}{2}}^T}$. These definitions are independent of the choice of the square-root matrix (see \cref{lemma:SPD-manifold-square-root}), and in practice, the unique symmetric positive definite square-root is used.

Given any function $f: \cS_{++} \rightarrow \R$, its Riemannian gradient at $X$ is
\begin{align*}
    {\rm grad} f(X) = X \nabla f(X) X \in T_X \cS_{++},
\end{align*}
where $\nabla f(X)$ is the Euclidean gradient.
To move along a tangent vector $\sigma \in T_X \cS_{++}$ for unit time along the geodesic, we use the exponential map:
\begin{align*}
    {\rm exp}_X(\sigma) = X^{\frac{1}{2}}\exp(X^{-\frac{1}{2}}\sigma X^{-\frac{1}{2}^T}) X^{\frac{1}{2}^T},
\end{align*}
with inverse given by the logarithmic map ${\rm exp}_{X}^{-1}(Y) = X^{\frac{1}{2}}\log(X^{-\frac{1}{2}}Y X^{-\frac{1}{2}^T}) X^{\frac{1}{2}^T}$. 
Gradient descent on $f$ with step size $\Delta t$ solves
\begin{align*}
    X_{n+1} = \argmin_X\Bigl\{g_{X_n}( {\rm exp}_{X_n}^{-1}(X), {\rm grad} f(X_n) ) + \frac{1}{2 \Delta t} {\rm d}(X, X_n)^2\Bigr\},
\end{align*}
this updates $X$ analytically according to 
\begin{equation}
    \label{eq:manifold_update}
    X_{n+1} = X_{n}^{\frac{1}{2}}\exp\Bigl(-\Delta tX_{n}^{\frac{1}{2}^T}\nabla f(X_{n}) X_{n}^{\frac{1}{2}}\Bigr) X_{n}^{\frac{1}{2}^T}.
\end{equation}
Take $f(C_k)=\mathrm{KL}\big[ \rho_a^{\rm GM} \Vert \pi \big]$. Its gradient is given by $\nabla f(C_k) = \frac{w_k}{2}C_k^{-\frac{1}{2}^T}E_k(C_k)C_k^{-\frac{1}{2}}$, where $E_k$ is defined in \cref{eq:Ek}. Then~\cref{eq:manifold_update} is the form used for the covariance update in \cref{eq:gm-ngf-update_C} up to a constant $\frac{w_k}{2}$. This also corresponds to mirror descent:
\begin{align*}
    X_{n+1} = \argmin_X\Bigl\{\langle \nabla f(X_n), X - X_n\rangle +  \frac{1}{\Delta t}D_{\phi}(X, X_n)\Bigr\},
\end{align*}
where $\langle\cdot,\cdot\rangle$ denotes the Frobenius inner product. The Bregman divergence is defined as $D_{\phi}(X,Y) = \phi(X) - \phi(Y) - \langle \nabla \phi(Y), X-Y \rangle$ with generator function $\phi(X) = {\rm tr}(S_{n}\log S_{n} - S_{n})$, where $S_n = X_n^{-\frac{1}{2}} X X_n^{-\frac{1}{2}^T}$. Note here $\phi$ is independent of the choice of square-root decomposition of $X$. The derivation uses the fact that $\nabla \phi(X) = X_n^{-1} \exp^{-1}_{X_n}(X) X_n^{-1}$.

Beyond the covariance update, the overall algorithm described in \cref{ssec:algorithm} can be interpreted as a mirror descent scheme. This provides a new perspective for studying Gaussian mixture variational inference, including properties such as positivity preservation, affine invariance, and convergence. We formalize this connection in the following theorem, with the proof in \cref{Proof:manifold}.
\begin{theorem}
\label{theorem-manifold}
For the Gaussian mixture parameter vector 
$$a =[m_1, ..., m_K, C_1, ..., C_K, w_1, ..., w_K],$$
we define the Bregman divergence $D_{\phi_n}(x,y) = \phi_n(x) - \phi_n(y) - \nabla \phi_n(y)^T (x-y)$
associated with the generator function
\begin{align*}
\phi_n(a) = \sum_{k=1}^{K} \left( \frac{w_{n,k}}{2} {\rm tr}(S_{n,k}\log S_{n,k} - S_{n,k}) + \frac{w_{n,k}}{2}m_k^T C_{n,k}^{-1} m_k +  w_k \log w_k \right),
\end{align*}
where $w_{n,k}=w_k(t_n),m_{n,k}=m_k(t_n),C_{n,k}=C_k(t_n)$ and $S_{n,k} = C_{n,k}^{-\frac{1}{2}} C_k  C_{n,k}^{-\frac{1}{2}}$.
The proposed Gaussian mixture variational inference algorithm, with update rules \cref{eq:gm-ngf-update_C,eq:gm-ngf-update_mw,eq:adaptive-dt}, can be written in the form of mirror descent:
\begin{equation}
\begin{split}
\label{eq:mirror_descent}
a_{n+1} = \argmin_{a}\left\{ \left\langle \nabla_{a}{\rm KL}[\rho_a^{\rm{GM}} \Vert \pi]|_{a_n}, a-a_n\right\rangle + \frac{1}{\Delta t_n}D_{\phi_n}(a,a_n) \right\} \quad{\rm s.t.} \quad \sum_{k=1}^{K} w_k = 1.
\end{split}
\end{equation}

Furthermore, consider an arbitrary invertible affine mapping $\psi: \theta \mapsto \tilde{\theta} = T\theta + d$. Let the transformed posterior distribution be $\widetilde{\pi} = \psi\#\pi$, with potential $\widetilde{\Phi}_R(\theta) = \Phi_R(T^{-1}(\theta - d))$. Then the algorithm remains invariant, with $\rho_{\tilde{a}}^{\rm GM} = \psi\#\rho_a^{\rm GM}$, namely the parameters and time step transform as follows:
    \begin{align*}
        \widetilde{m}_k(t)  = T m_k(t) + d,\quad
        \widetilde{C}_k(t)  = T C_k(t) T^T,\quad \widetilde{w}_k(t)  = w_k(t), \quad
        \Delta \widetilde{t}_n  = \Delta t_n.
    \end{align*}
\end{theorem}

\begin{newremark}
The potential function $\phi_n(a)$ in the mirror descent framework combines the covariance terms, the mean terms, and the weight terms, which induce the update of the three parts respectively. The resulting scheme is affine invariant at the discrete level, making the algorithm robust to coordinate transformations. In practice, as the result is independent of the decomposition method, Cholesky factorization is employed for efficient computation of matrix operations involving covariance matrices.
\end{newremark}

\section{Implementation details}
\label{sec:implementaion}

\subsection{Scheduler}
\label{ssec:scheduler}
The scheduler $\eta(t)$ in \cref{eq:adaptive-dt} 
plays a crucial role in improving convergence, as demonstrated by both our theoretical analysis (See~\cref{thm:Sigma_n-noise}) and the experimental results that follow. The asymptotic convergence analysis uses a Robbins-Monro-type scheduler. In the finite-horizon experiments, we instead use
the following stable cosine decay, a variant of the schedulers proposed in~\cite{loshchilov2016sgdr,li2025functional}:
\begin{equation}
\label{eq:stable_cos_scheduler}
    \eta(t_n) = \begin{cases}
        1, & n \leq \frac{N}{2}, \\ 
        \eta_{\rm min} + \frac{1 - \eta_{\rm min}}{2} \left(1 + \cos\left(2\pi\left(\frac{n}{N} - \frac{1}{2}\right)\right)\right), & n > \frac{N}{2},
    \end{cases}
\end{equation}
where $\eta_{\rm min} < 1$, and $N$ is the total number of iterations. This schedule should be understood as a practical finite-horizon choice: it keeps the step size large during the early phase and gradually reduces it near convergence.
If asymptotic convergence is desired, the finite-horizon schedule can be followed by a Robbins-Monro tail. A comparison with other schedulers is presented in \cref{sec:sensitivity}. In our experiments, the stable cosine schedule performs slightly better, although the differences are not substantial.
% The first half of the iterations corresponds to the constant phase, while the second half corresponds to the decay phase.

\subsection{Annealing-based initialization}
\label{ssec:annealing}

For multimodal posteriors or posteriors with heavy tails (See cases A and C in \cref{ssec:model-problems}), annealing can significantly improve exploration and approximation quality. Incorporating annealing into the variational framework is straightforward: the KL-divergence \cref{eq:KL} is modified to:
$$
\min_{\rho_a} \ \mathbb{E}_{\rho_a}[\log \rho_a] + \frac{1}{T} \mathbb{E}_{\rho_a}[\Phi_R],
$$
where $T \ge 1$ controls the trade-off between the entropy and cross-entropy terms. Higher temperatures encourage exploration, helping Gaussian components move between modes instead of getting trapped in one.

In the present work, we employ an exponential annealing scheduler to reduce the temperature parameter $T$ from an initial value $T_{\text{start}}$ to $1$ over $N_\alpha$ iterations, with $T_n = T_{\text{start}}^{\frac{N_\alpha-n}{N_\alpha-1}}$ for $n=1,\cdots,N_\alpha$. The resulting distribution after the final iteration is used as the initial condition of GMBBVI. The initial value $T_{\text{start}}$ is set such that the natural gradient of the entropy term dominates that of the cross-entropy term:
$$
\left\| \frac{1}{T_{\text{start}}} \widetilde{\nabla}_{m} \mathbb{E}_{\rho_a}[\Phi_R] \right\| \le \alpha \cdot \| \widetilde{\nabla}_{m} \mathbb{E}_{\rho_a}[\log \rho_a] \|,
$$
where $\alpha \in (0,1)$ is a small constant, and $\widetilde{\nabla}_{m} $ denotes the natural gradient with respect to means $m_1,\cdots,m_K$, using the same block-diagonal approximation of the Fisher information matrix as in \cref{sec-Natural gradient variational inference}. This ensures early updates are entropy-driven for robust exploration. Here, we do not consider natural gradients based on the weights and covariance, as the primary goal during the annealing phase is to increase the separation between component means. These natural gradients are estimated in a derivative-free manner similar to \cref{eq:gm-ngf}.

\section{Experimental results}
\label{sec:experiments}
In this section, we present numerical studies of the proposed GMBBVI method for sampling complex distributions of unknown parameters or fields\footnote{Our code is available online for reproducibility:
\url{https://github.com/PKU-CMEGroup/InverseProblems.jl/tree/master/Derivative-Free-Variational-Inference}.}. Specifically, we consider the following examples:
\begin{enumerate}
    \item  {Multi-dimensional model problems (up to 50 dimensions):} We use these problems to demonstrate that GMBBVI is robust with respect to distributions featuring multiple modes, infinitely many modes, and modes with narrow and curved shapes.  
    Additionally, on Neal's funnel sampling problem, we benchmark BBVI against WALNUTS~\cite{Nawaf2025WALNUTS}, a state-of-the-art derivative-based sampling approach for this case, to highlight the strengths and limitations of our method.
    \item   Darcy flow problem: We consider the inverse problem of recovering the permeability field in a Darcy flow equation. The problem is structured to exhibit symmetry, resulting in two modes in the posterior. We demonstrate that GMBBVI effectively captures both modes, illustrating its potential for tackling multimodal problems in large-scale, moderate-dimensional applications.
\end{enumerate}
For all experiments, we set $\Delta t_{\rm max} = \beta = 0.9$ in \cref{eq:adaptive-dt}, adopt the stable cosine decay scheduler mentioned in \cref{eq:stable_cos_scheduler} decreasing from $1$ to $0.1$, and use $J=4 N_\theta$ samples for each Gaussian integral in \cref{eq:gm-ngf-repar}. 
%A parameter sensitivity study is presented in \cref{sec:parameter}.

\subsection{Multi-dimensional model problems}
\label{ssec:model-problems}
In this subsection, we first investigate several classical 2D sampling problems, together with their modified higher-dimensional counterparts in 10 and 50 dimensions.
For the 2D setting, we consider the following problems:

\begin{enumerate}[label=Case \Alph*, left=0.2cm]
\item: The target distribution is a Gaussian mixture with $10$ modes (See \cref{fig:GMBBVI-trials}).

\item : The target distribution has a circular geometry \cite[Appendix E]{chen2024efficient} and contains infinitely many modes. It is defined by 
\begin{align*}
\Phi_R(\theta) = \frac{1}{2}\F(\theta)^T\F(\theta), \textrm{ where } \F(\theta) = \frac{y - (\theta_{(1)}^2 + \theta_{(2)}^2)}{0.3}
\textrm{ and } 
y = 1.
\end{align*}
\item : 
The target distribution is based on the Rosenbrock function, which exhibits the characteristic ``banana-shaped" geometry~\cite{goodman2010ensemble}. It is given by
\begin{align*}
\Phi_R(\theta) = \frac{1}{2}\F(\theta)^T\F(\theta),\textrm{ where}\, \F(\theta) = 
\frac{1}{\sqrt{10}}\Bigl(y - \begin{bmatrix}
10(\theta_{(2)} -  \theta_{(1)}^2)\\
\theta_{(1)}
\end{bmatrix}\Bigr), 
y= 
\begin{bmatrix}
0\\
1
\end{bmatrix}.
\end{align*}
\end{enumerate}
We then modify these sampling problems to an $N_\theta$-dimensional setting ($N_\theta = 10,50$) by introducing $N_\theta - 2$ additional variables $\theta^c$. For Case A, the latter $N_\theta - 2$ dimensions are modeled as independent Gaussian random variables, each with a mean of $m_i \sim \mathcal{N}(0, 1)$ and a variance of 1. For Cases B and C, the reference density is defined with
\[
\Phi_R(\theta,\theta^c) = \frac{1}{2}\F(\theta)^T\F(\theta) + \frac{1}{2} (\theta^c-K\theta)^T (\theta^c-K\theta),
\]
where $\theta \in \R^2, \theta^c \in \R^{N_\theta-2}$, and $K \in \R^{(N_\theta-2) \times 2}$ is an all-ones matrix. These high-dimensional target densities are constructed so that the marginal densities of $\theta$ are analytically tractable and coincide exactly with the corresponding 2D target density.

\begin{figure}[htbp]
    \centering
    \includegraphics[scale=0.25]{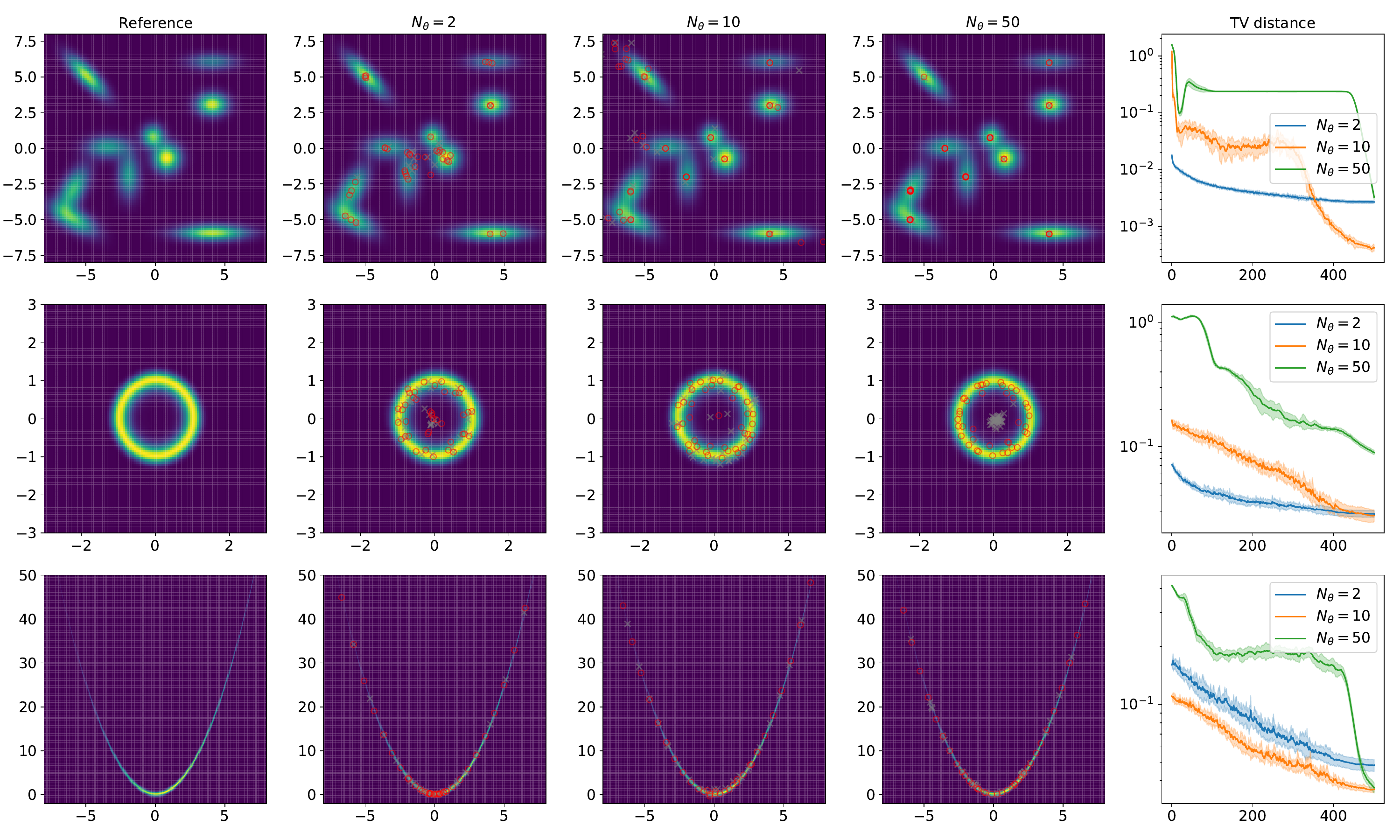}
    \caption{Multi-dimensional model problems: Cases A to C are arranged from the top row to the bottom row. Each panel shows the reference marginal density along with the GMBBVI estimates for problem dimensions 2, 10, 50 (from left to right). The projected means of each Gaussian component are marked by red circles, and the projected means at iteration $0$ (after annealing) are marked by grey crosses. The fifth panel displays the total variation distance between the reference marginal density and the estimated marginal densities across the iterations, with shaded area representing the standard deviation computed from 10 independent trials.}
    \label{fig:GMBBVI-trials}
\end{figure}

We apply GMBBVI with $K = 40$ modes, initialized with means $m_i\sim \N(0,I)$ and covariances $C_i = I$, and run the algorithm for 500 iterations. For Cases A and C, the annealing strategy described in \cref{ssec:annealing} with $N_\alpha = 500$ and $\alpha=0.1$ is employed to address multimodality and heavy-tailed behavior.
The resulting marginal densities and convergence in terms of total variation (TV) distance are shown in \cref{fig:GMBBVI-trials}, with each row corresponding to one case (Cases A-C). 
For the cases considered, GMBBVI effectively captures multiple modes and provides a good approximation to the target distribution, converging within $\mathcal{O}(10^2)$ iterations to achieve a TV distance below $0.1$ for moderately high-dimensional problems.
For comparison, we also apply some other Gaussian mixture variational inference methods to the above examples, demonstrating the efficiency and approximation accuracy of GMBBVI; see \cref{ssec:comparison-among-VI}.
%However, in higher-dimensional settings, GMBBVI may fail to capture all modes or may yield a less accurate approximation of the target distribution. 

Next, we consider a challenging benchmark known as Neal's funnel~\cite{neal2003slice}, defined as
\[
\theta_{(1)} \sim \mathcal{N}(0, 9),\quad \theta_{(i)} \mid \theta_{(1)} \sim \mathcal{N}(0, e^{\theta_{(1)}}), \quad i = 2, \dots, N_{\theta}. 
\]
The joint distribution exhibits a funnel-like geometry, with $\theta_{(1)}$ forming the funnel axis. This distribution can be viewed as a model problem for Bayesian hierarchical models. The multiscale behavior of the funnel makes it particularly challenging; we expect our GMBBVI to adapt to the local geometry by using different Gaussian components. We also employ the recently proposed WALNUTS \cite{Nawaf2025WALNUTS}, which is an improved variant of the No-U-Turn sampler~\cite{hoffman2014no}  for this benchmark.

We apply GMBBVI with $K = 40$ modes, initialized with means $m_i\sim \N(0,I)$ and covariances $C_i = I$, and run the algorithm for 2000 iterations. For  WALNUTS, we adopt the parameter settings from \cite{Nawaf2025WALNUTS}, utilizing WALNUTS-R2P and $1000$ warm-up iterations with $H=0.3$ and $\delta=0.3$ for a total $16000$ accepted proposals.
Comparative results, including marginal densities and several convergence metrics, are presented in  \cref{fig:Funnel_Comparison}, where each row corresponds to a dimensional setting ($N_\theta = 2,10, 50$). 
The results indicate that GMBBVI achieves accuracy comparable to WALNUTS in low-dimensional settings ($N_\theta = 2, 10$), while offering a clear computational speed advantage. However, in higher dimensions ($N_\theta = 50$), GMBBVI fails to provide reliable marginal variance estimates compared with WALNUTS. This limitation arises because  Gaussian components become trapped and oscillate in transverse directions near the narrow neck of the funnel, making exploration along the axial direction inefficient.

\begin{figure}[htbp]
    \centering    
    \includegraphics[scale=0.25]{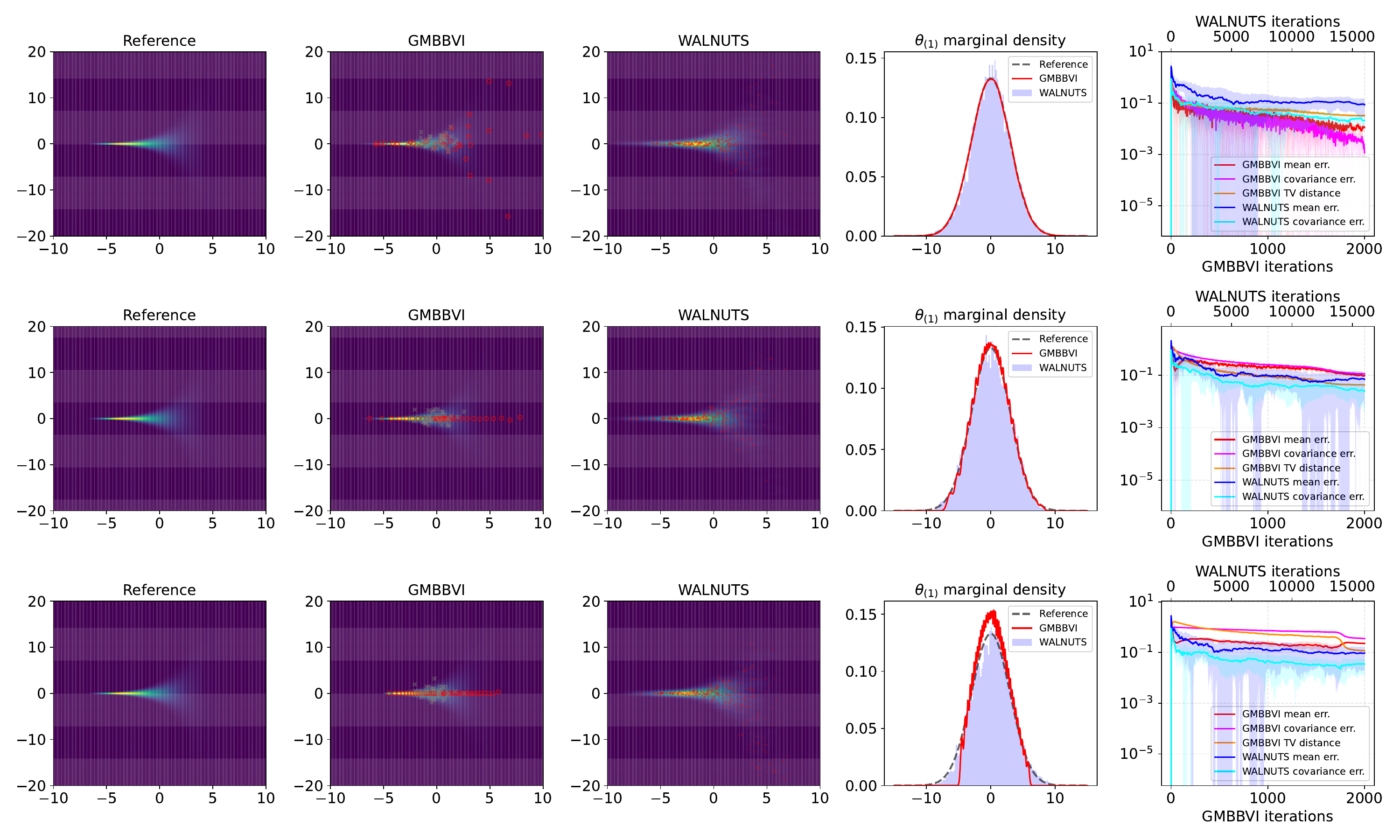}
    \caption{
    Neal's funnel model problem: Dimensions $2$, $10$, and $50$ are arranged from the top row to the bottom row. Each panel shows the reference marginal density of $(\theta_{(1)},\theta_{(2)})$ along with the GMBBVI and WALNUTS estimates (from left to right). For GMBBVI, the projected means of each Gaussian component are marked by red circles, and the initializations are marked by grey crosses. For WALNUTS, red dots denote the last 5000 particles, and the marginal density is visualized using kernel density estimation. 
    The fourth panel displays the reference marginal density of $\theta_{(1)}$ along with the GMBBVI and WALNUTS estimates.
    The fifth panel reports convergence metrics for $\theta_{(1)}$, including its mean and variance, for both methods, as well as the total variation distance between the reference marginal density and the GMBBVI estimates of $(\theta_{(1)},\theta_{(2)})$, with the shaded area representing the standard deviation computed from 10 independent trials.} \label{fig:Funnel_Comparison} 
\end{figure}

Finally, we demonstrate that GMBBVI is robust with respect to the initial condition, the scheduler $\eta$, the annealing parameter, and a sufficiently large number of modes $K$. Detailed sensitivity results are presented in \Cref{sec:sensitivity}.

\subsection{Darcy flow problem}
\label{ssec:Darcy}
In this subsection, we evaluate our proposed GMBBVI algorithm on a Bayesian inverse problem arising from the Darcy flow equation on the two-dimensional spatial domain $D=[0,1]^2$. This equation characterizes the pressure field $p(x)$ in a porous medium governed by a positive permeability field $a(x,\theta)$:
\begin{align}
    \label{eq:Darcy-2D}
    -\nabla \cdot (a(x, \theta) \nabla p(x)) &= f(x), \quad x\in D,
\end{align}
subject to homogeneous Dirichlet boundary conditions $p(x)=0$ on $\partial D$ for simplicity. The fluid source field $f$ is defined piecewise in the $x_2$ direction as follows: it takes value $1000$ for $x_2 \leq \frac{4}{6}$, $2000$ for $\frac{4}{6} < x_2 \leq \frac{5}{6}$, and $3000$ for $\frac{5}{6} < x_2 \leq 1.$
The logarithm of the permeability field $\log a(x_1,x_2)$ is parameterized via a Karhunen-Lo\`eve expansion:
\begin{equation}
    \label{eq:KL-expan}
    \log a(x_1,x_2) = \sum_{l} \theta_l \sqrt{\lambda_l} \phi_l(x_1,x_2).
\end{equation}
Here the summation is taken over $l=(l_1,l_2)\in \Z^{0+}\times \Z^{0+} \setminus \{(0,0)\}$, $\{\lambda_l\}$ denotes the eigenvalues arranged in descending order as
\begin{equation}
    \lambda_l = (\pi^2(l_1^2 + l_2^2) + \tau^2)^{-d},
\end{equation}
where $d = 2.0$ and $\tau = 3.0$ are parameters governing the decay rate of the eigenvalues. The corresponding eigenfunctions  $\{\phi_l\}$ are defined as
\begin{equation}
    \phi_l(x_1,x_2) = 
    \begin{cases}
        \sqrt{2} \cos(\pi l_1 x_1), & \text{if } l_2 = 0, \\
        \sqrt{2} \cos(\pi l_2 x_2), & \text{if } l_1 = 0, \\
        2 \cos(\pi l_1 x_1) \cos(\pi l_2 x_2), & \text{otherwise}.
    \end{cases}
\end{equation}

For numerical implementation, we truncate the Karhunen-Lo\`eve expansion~\cref{eq:KL-expan} to its first $N_\theta=32$ terms and solve the Darcy flow equation~\cref{eq:Darcy-2D} using the finite difference method, where the computational domain is discretized by an $80 \times 80$ uniform grid. The observations $y_{\rm obs}$ in the inverse problem are selected as the symmetric pressure sum $\frac{1}{2}[p(x_1,x_2)+p(1-x_1,x_2)]$ evaluated at $28$ equidistant points in the left region of the domain (see \cref{fig:Darcy-2D-observation}), thereby inducing the forward mapping $\G:\theta\in \R^{32}\mapsto y\in \R^{28}$. The Bayesian inverse problem aims to sample from the posterior distribution of $\theta$, given by
\[
    \rho_{\rm post}(\theta)\propto \exp \left(-\frac{1}{2\sigma_\eta^2}\|y_{\rm obs}-\G(\theta)\|^2-\frac{1}{2\sigma^2_0}\|\theta\|^2\right).
\]
Here we adopt the prior distribution of $\theta$ as $\N(0,\sigma_0 ^2 I_{32})$, and $y_{\rm obs}$ represents the observations corrupted by noise $\eta \sim \N(0,\sigma_\eta ^2 I_{28})$, with $\sigma_0=5.0$ and $\sigma_\eta=0.25$. Due to the symmetric nature of our observations and fluid source field, symmetric permeability fields induce identical observations, rendering the aforementioned posterior distribution at least bimodal.

\begin{figure}
    \centering
    \includegraphics[width=0.4\linewidth]{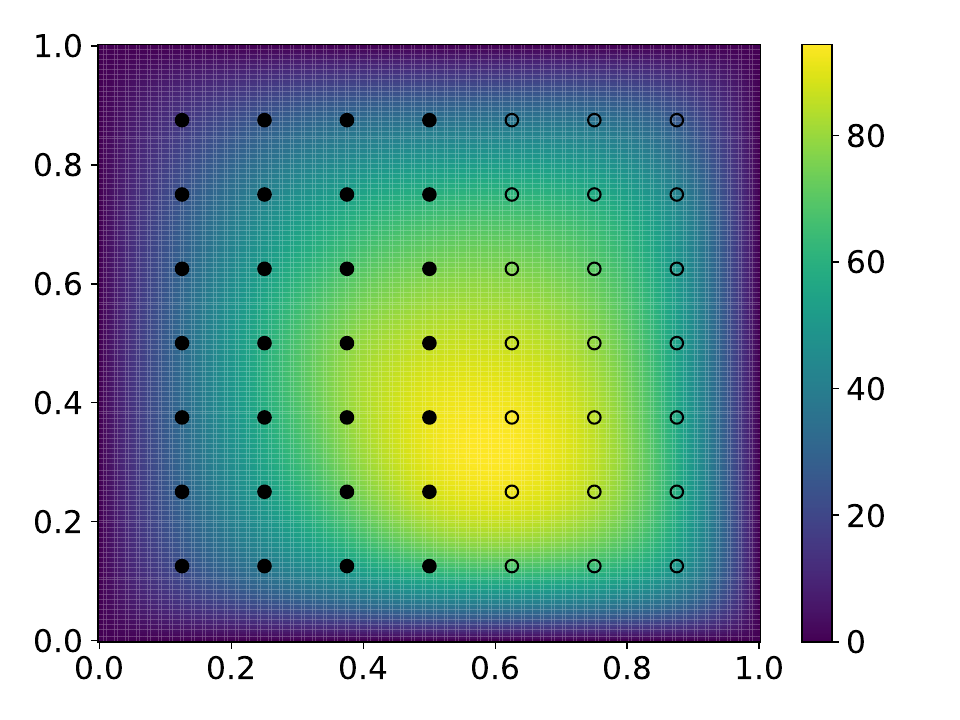}
    \caption{Darcy flow problem: The pressure field and symmetry observations at 28 equidistant points (solid black dots). Their mirroring points are marked (empty black dots).}
    \label{fig:Darcy-2D-observation}
\end{figure}
In our inverse problem setup, we randomly select $\theta_{\rm ref}$ according to the prior distribution, obtain the corresponding permeability field, and solve the Darcy flow equation~\cref{eq:Darcy-2D} on a grid refined by a factor of three. Additive noise following $\N(0,\sigma_\eta^2 I_{28})$ is incorporated into the observation $y_{\rm obs}$. 

We run GMBBVI for 500 iterations, with $K=5$ modes where each mean is initialized according to $\N(0,\sigma_0 ^2I)$ and each covariance is set to $\sigma_0^2 I$. Upon completion of the iterations, we compare the logarithm of permeability fields corresponding to each $m_k$ against the true and mirror permeability fields to partition them into two groups. The weighted means of $m_k$ within each group are computed to reconstruct the permeability fields, as presented in \cref{fig:Darcy-logk}. The evolution curves of various metrics during iterations are shown in \cref{fig:Darcy-convergence}, illustrating that GMBBVI converges in around 100 iterations. Additionally, \cref{fig:Darcy-density} displays the marginal distributions of the Gaussian mixture on the first 16 dimensions. These results demonstrate that our algorithm GMBBVI achieves rapid convergence without requiring derivatives of $\G$, while successfully capturing the multimodal nature of the posterior distribution.

\begin{figure}
    \centering
    \includegraphics[width=0.85\linewidth]{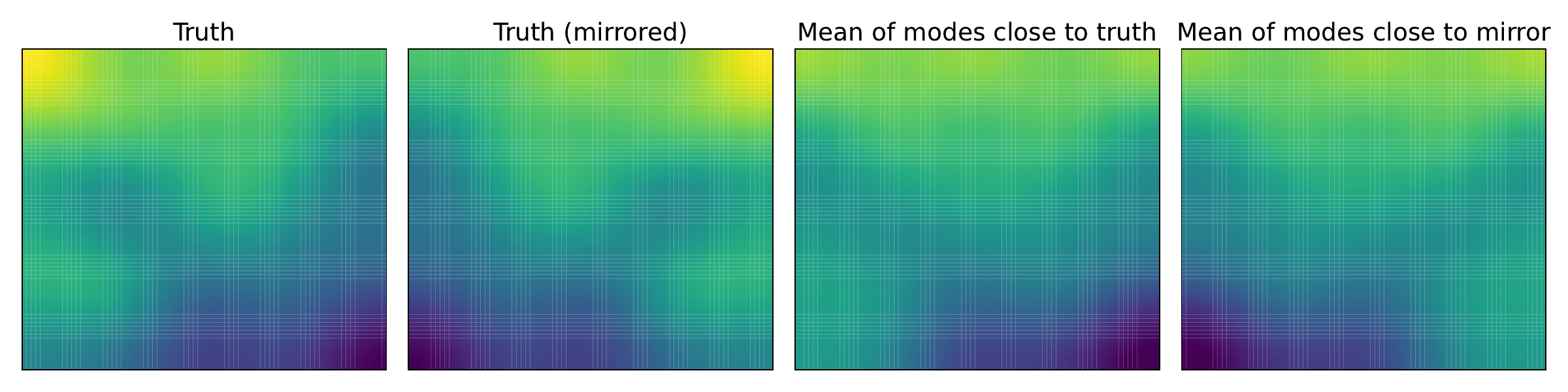}
    \caption{Darcy flow problem: The true initial permeability field (leftmost), its mirrored field (middle-left), recovered permeability fields by modes close to truth (middle-right) and close to mirrored truth (rightmost).}
    \label{fig:Darcy-logk}
\end{figure}

\begin{figure}
    \centering
    \includegraphics[width=0.95\linewidth]{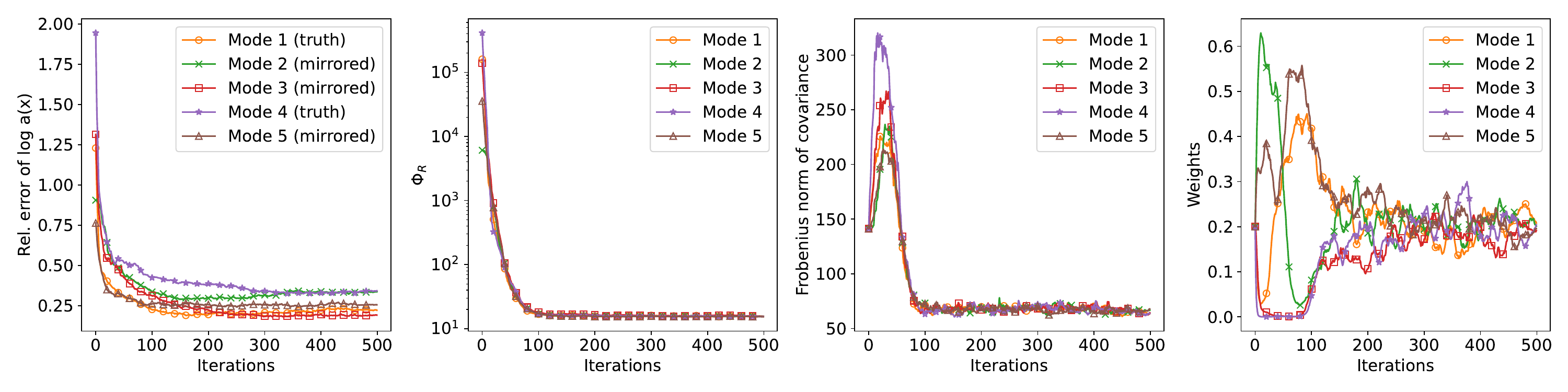}
    \caption{The relative errors of the logarithm of permeability field, the optimization errors $\Phi_R (m_k )$, the Frobenius norm $\|C_k \|_F $, and the Gaussian mixture weights $w_k$ (from left to right) for different modes over GMBBVI iterations.}
    \label{fig:Darcy-convergence}
\end{figure}

\begin{figure}
    \centering
    \includegraphics[width=0.85\linewidth]{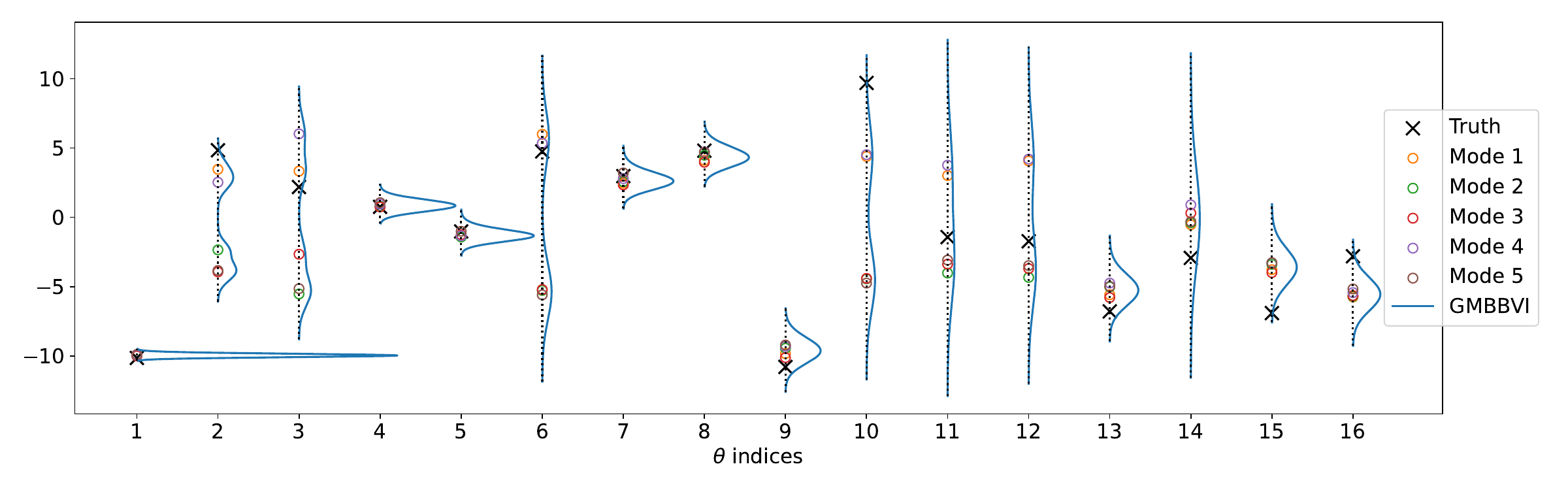}
    \caption{ The true KL expansion parameters $\theta_{(i)}$ (black crosses), and mean estimations of $\theta_{(i)}$ for each mode (circles) and the associated marginal distributions obtained by GMBBVI at the 500th iteration. Only the first 16 dimensions are plotted.}
    \label{fig:Darcy-density}
\end{figure}

\section{Conclusions}
\label{sec:conclusions}
In this paper, we investigate stability issues in Gaussian mixture black-box variational inference methods. We propose an adaptive exponential time integrator and demonstrate its superior stability and efficiency in approximating moderately high-dimensional probability distributions through both theoretical analysis and empirical validation.

Several promising directions for future research remain. On the algorithmic side, the proposed exponential integrator and its connections to manifold optimization and mirror descent could also be extended to other gradient-based sampling methods involving covariance matrices or other constrained manifolds.
On the theoretical side, a more detailed analysis of the approximation properties and optimization guarantees of Gaussian mixtures and related variational approaches, for example for log-concave target densities, would provide deeper insight into their practical performance and limitations.

\section*{Acknowledgments} 
B. Che, D.Z. Huang, X. Mao, and W. Wang acknowledge funding support from National Key R\&D Program of China 2025YFA1018700, 
National Natural Science Foundation of China (No.12471403, 62595771, and 12288101), the Fundamental Research Funds for the Central Universities of China, and the support of the high-performance computing platform of Peking University.

\appendix
\section{Derivations and Proof of the Convergence}
\label{Proof:convergence}

\textbf{Derivation of \cref{eq:Gaussian-matrix} and \cref{eq:LL-h}.} For the Gaussian case that $\Phi_R(\theta)=\frac{1}{2}(\theta-m_{\star})^T
C_{\star}^{-1}(\theta-m_{\star})$, and following the definitions in \cref{eq:fk-reparameter,eq:Sigma_v}, we have 
\begin{align}
    f(t_n,\theta)
    &=-\frac{1}{2}\theta^T\theta+\frac{1}{2}(L(t_n)\theta+m(t_n)-m_{\star})^TC_{\star}^{-1}(L(t_n)\theta+m(t_n)-m_{\star})+\text{const} \nonumber\\
    &=\frac{1}{2}\theta^T(\LL_n^T\LL_n-I)\theta+v_n^T\LL_n\theta+\frac{1}{2}\|v_n\|^2+\text{const}\label{eq:f_t_n_theta}.
\end{align}  Then we have 
\begin{align*}
    &\E_{\N}\Bigl[\theta \bigl(f - \E_{\N}[f] \bigr)\Bigr] 
    = \E_{\N}[\nabla_\theta f]  = \LL_n^T v_n,\\
    &\E_{\N}\Bigl[\theta\theta^T \bigl(f - \E_{\N}[f] \bigr)\Bigr] 
    = \E_{\N}[\nabla_\theta\nabla_\theta f]  = \LL_n^T\LL_n-I.
\end{align*}
Using these calculations and the update schemes in \cref{eq:gm-ngf-update_mw,eq:gm-ngf-update_C}, we obtain
\begin{align*}
 v_{n+1}&=v_n -\Delta t_n \LL_n\LL_n^T v_n=(I-\Delta t_n\Sigma_n)v_n \\
 \Sigma_{n+1}&=\LL_n e^{-\Delta t_n(\LL_n^T\LL_n-I)}\LL_n^T= \Sigma_n \LL_n^{-T}e^{\Delta t_n(I-\LL_n^T \LL_n)} \LL_n^T \\
 &= \Sigma_n e^{\Delta t_n\LL_n^{-T}(I-\LL_n^T \LL_n) \LL_n^T} = \Sigma_n e^{\Delta t_n(I-\Sigma_n)}=h_n(\Sigma_n),
\end{align*} where $h_n$ is defined in \cref{eq:LL-h}. That finishes the derivation of  \cref{eq:Gaussian-matrix} and \cref{eq:LL-h}.

\textbf{Derivation of \cref{eq:Gaussian-mixture-matrix}.} For the well-separated Gaussian mixture case where $f_k$ in \cref{eq:fk-reparameter-approx} is approximated by
\begin{equation}
    \begin{split}
    f_k(t_n, \theta) &= -\frac{1}{2}\theta^T\theta + \frac{1}{2}(L_{k}(t_n)\theta + m_{k}(t_n) - m_{\star,k})^TC_{\star,k}^{-1}(L_{k}(t_n)\theta + m_{k}(t_n) - m_{\star,k})  - \frac{1}{2}\log\frac{|C_{n,k}| }{|C_{\star,k}|}   + \log\frac{w_{n,k}}{w_{\star,k}}  \\
    &= \frac{1}{2}\theta^T(\LL_{n,k} ^T\LL_{n,k}-I)\theta+v_{n,k}^T \LL_{n,k}\theta+\frac{1}{2}\|v_{n,k}\|^2  - \frac{1}{2}\log|\Sigma_{n,k}|  + \log \frac{w_{n,k}}{w_{\star,k}},
    \end{split}
\end{equation} 
we have 
\begin{align*}
    &\E_{\N}\Bigl[\theta \bigl(f_k - \E_{\N}[f_k] \bigr)\Bigr] 
    = \E_{\N}[\nabla_\theta f_k]  = \LL_{n,k}^T v_{n,k},\\
    &\E_{\N}\Bigl[\theta\theta^T \bigl(f_k - \E_{\N}[f_k] \bigr)\Bigr] 
    = \E_{\N}[\nabla_\theta\nabla_\theta f]  = \LL_{n,k}^T\LL_{n,k}-I,\\
    &\E_{\N}[f_k]  = \frac{1}{2} \bigl(\textrm{tr}[\Sigma_{n,k}-I] +  v_{n,k}^Tv_{n,k} -  \log|\Sigma_{n,k}|\bigr)  +\log  \frac{w_{n,k}}{w_{\star,k}}.
\end{align*}
Using these calculations and the update schemes in \cref{eq:gm-ngf-update_mw,eq:gm-ngf-update_C}, we obtain
\begin{align*}
 v_{n+1,k}&= 
 v_{n,k} -\Delta t_{n} \LL_{n,k}\LL_{n,k}^T v_{n,k}=(I-\Delta t_{n}\Sigma_{n,k})v_{n,k}, \\
 \Sigma_{n+1,k}&=
 \LL_{n,k} e^{-\Delta t_{n}(\LL_{n,k}^T\LL_{n,k}-I)}\LL_{n,k}^T= \Sigma_{n,k} \LL_{n,k}^{-T}e^{\Delta t_{n}(I-\LL_{n,k}^T \LL_{n,k})} \LL_{n,k}^T \\
 &= \Sigma_{n,k} e^{\Delta t_{n}\LL_{n,k}^{-T}(I-\LL_{n,k}^T \LL_{n,k}) \LL_{n,k}^T} = \Sigma_{n,k} e^{\Delta t_{n}(I-\Sigma_{n,k})},\\
 \log \widehat{w}_{n+1,k} &= \log w_{n,k} - \Delta t_n \Bigl( \frac{1}{2} \bigl(\textrm{tr}[\Sigma_{n,k}-I] +  v_{n,k}^Tv_{n,k} -  \log|\Sigma_{n,k}|\bigr)  +\log  \frac{w_{n,k}}{w_{\star,k}} \Bigr) .
\end{align*}
From the last equation, we have
\[
\widehat{w}_{n+1,k}  = w_{n,k} \bigl(\frac{w_{\star,k}}{w_{n,k}}\bigr)^{\Delta t_n}e^{-\frac{\Delta t_n}{2} \bigl(\textrm{tr}[\Sigma_{n,k}-I] +  v_{n,k}^Tv_{n,k} -  \log|\Sigma_{n,k}|\bigr)},
\]
That finishes the derivation of  \cref{eq:Gaussian-mixture-matrix}.

\subsection{Proof of \texorpdfstring{\cref{theorem-Sigma_n-no-noise}}{theorem 1}}
\label{proof_of_theorem-Sigma_n-no-noise}

\begin{proof}[Proof of \cref{theorem-Sigma_n-no-noise}]
\label{proof:Sigma_n-no-noise}

\textbf{We first consider the convergence of $\{\Sigma_n\}$.} Let $\lambda_n^i (1 \leq i \leq N_\theta)$ denote the eigenvalues of $\Sigma_n$, ordered increasingly. Writing the eigen decomposition as $\Sigma_n = P^T\Lambda_nP$, the update rule gives 
$$\Sigma_{n+1} = P^T h_n(\Lambda_n)P, \quad \textrm{with} \quad \Lambda_{n+1} = h_n(\Lambda_n) \quad  \textrm{diagonal}.$$ 
Therefore, these eigenvalues evolve according to the iteration $h_n(x)$, and are reordered after each step. 

We outline several properties of $h_n(x) = xe^{\Delta t_n (1 - x)}$ for scalar $x > 0$ with $0 < \Delta t_n \leq 1$.
First, its derivative is 
    $h'_n(x) = e^{\Delta t_n(1- x)}(1 - \Delta t_n x)$,
    so $h_n(x)$ is increasing up to its maximum at $x = \frac{1}{\Delta t_n} \geq 1$, and decreasing thereafter, hence 
    \begin{equation}
        h_n(x)\leq h_n(\frac{1}{\Delta t_n}).
    \end{equation}
Second, for $x < 1$, $h_n(x) > x$ since $e^{\Delta t_n (1 - x)} > 1$,
    and $h_n(x) < 1$, since $h_n(x) < h_n(1) = 1$. 
For $ x > 1$, $h_n(x) < x$ since $e^{\Delta t_n (1 - x)} < 1$, and moreover, $|h_n(x) - 1| < x - 1.$ 
    If $h_n(x) > 1$, this is immediate; if $h_n(x) < 1$, the inequality follows the fact 
    $2 < x(1 + e^{1 - x}) \leq x(1 + e^{\Delta t_n (1 - x)}) \quad \textrm{with} \quad x > 1.$
    Therefore, for $x > 0$, we have 
     \begin{align} \label{eq:h_n-contraction-to-1}
       |h_n(x) - 1| \leq |x - 1|, 
    \end{align}
    with equality only at $x=1$.
This property implies that all eigenvalues move closer to $1$ at each iteration, including the largest and smallest eigenvalues.

\textbf{Step 1. Warm-up phase for large eigenvalues.}
We first show that after a warm-up phase of $N_1 = \mathcal{O}\Bigl(|\log \lambda_{\rm max}(\Sigma_0)|\Bigr)$ iterations, the adaptive time step remains uniformly bounded away from $0$, independent of initialization. By definition, 
$$
\Delta t_n = \min\bigl\{\frac{\beta}{\|\Sigma_n-I\|_2}, \Delta t_{\max}\bigr\} = \min\bigl\{\frac{\beta}{|\lambda_{n}^{N_\theta} - 1|},\frac{\beta}{|1 - \lambda_{n}^{1}|}, \Delta t_{\max}\bigr\}.
$$
Suppose $\lambda_{n}^{N_\theta} > 1$; otherwise, $\Delta t_n \geq \min\{\beta, \Delta t_{\rm max}\}$. We show that the
largest eigenvalue, $\lambda_{n+1}^{N_\theta} = \max_{1\leq i\leq N_\theta}\{h_n(\lambda_{n}^{i})\}$, decreases exponentially until it enters a bounded domain. Indeed,
\begin{align}
\label{eq:x^N_theta-warm-up-contraction}
\lambda_{n+1}^{N_\theta} \leq \max\{\frac{1}{\Delta t_{\rm max}} e^{\Delta t_{\rm max} - 1}, \frac{2}{\beta} e^{\frac{\beta}{2} - 1}, \lambda_{n}^{N_\theta} e^{-\beta}, \lambda_{n}^{N_\theta} e^{\frac{\beta}{2}-1}\}.
\end{align}
This bound reflects three possible cases: 
\begin{enumerate}
    \item If $\Delta t_n = \Delta t_{\rm max}$,  then  $h_n(x) \leq h_n(\frac{1}{\Delta t_n}) = \frac{1}{\Delta t_{\rm max}} e^{\Delta t_{\rm max} - 1}$.
    \item If $ \Delta t_{\rm max} > \Delta t_n$ and  $\Delta t_n >  \frac{\beta}{2}$, then $h_n(x) \leq h_n(\frac{1}{\Delta t_n}) \leq \frac{2}{\beta} e^{\frac{\beta}{2} - 1}$.
    \item If $ \Delta t_{\rm max} > \Delta t_n$ and $\Delta t_n < \frac{\beta}{2}$, then $\Delta t_n = \frac{\beta}{\lambda_{n}^{N_\theta} - 1}$ and $h_n(\lambda_n^i) = \lambda_n^ie^{-\beta\frac{\lambda_n^i-1}{\lambda_{n}^{N_\theta} - 1}}$. The maximum occurs either at $\lambda_{n}^{N_\theta}$ or at $\frac{\lambda_{n}^{N_\theta} - 1}{\beta}$. For the first case, $h_n( \lambda_{n}^{N_\theta} ) =  \lambda_{n}^{N_\theta}  e^{-\beta}$;
    For the second case, requiring $\frac{\lambda_{n}^{N_\theta} - 1}{\beta} \leq \lambda_{n}^{N_\theta}$, we have $h_n(\frac{\lambda_{n}^{N_\theta} - 1}{\beta}) = \frac{\lambda_{n}^{N_\theta} - 1}{\beta} e^{\frac{\beta}{\lambda_{n}^{N_\theta} - 1} - 1} \leq \lambda_{n}^{N_\theta} e^{\frac{\beta}{2} - 1}$.
\end{enumerate}
Therefore after 
\begin{align}
\label{eq:N_1}
   N_1 = \frac{\log \lambda_{\rm max}(\Sigma_0)}{\min\{\beta, 1- \frac{\beta}{2}\} } = \mathcal{O}\bigl(\bigl|\log \lambda_{\rm max}(\Sigma_0)\bigr|\bigr)
\end{align} iterations, we obtain for all $n\geq N_1$: 
\begin{align}
\label{eq:x_max_N_theta}
    \lambda_{n}^{N_\theta} \leq  \lambda_{\rm max}^{N_\theta} := \max\{\frac{1}{\Delta t_{\rm max}} e^{\Delta t_{\rm max} - 1}, \frac{2}{\beta} e^{\frac{\beta}{2} - 1}\},
\end{align}
and consequently,
\begin{align}
\label{eq:Delta_t_min}
    \Delta t_n \geq \Delta t_{\rm min} := \min\bigl\{\frac{\beta}{\lambda_{\rm max}^{N_\theta} - 1}, \beta , \Delta t_{\max}\bigr\}.
\end{align}

\textbf{Step 2. Convergence phase for large eigenvalues.}
We then prove that after a convergence phase with $N_2 = \mathcal{O}\bigl(\log\frac{1}{\epsilon}\bigr)$ iterations, the largest eigenvalue satisfies
$\lambda_{n}^{N_\theta} \leq 1 + \epsilon$.
From \cref{eq:h_n-contraction-to-1}, $h_n(x)$ for $x > 1$ is a contraction: 
\begin{align*}
h_n(x) - h_n(1) = h'_n(\xi) (x - 1) = e^{\Delta t_n(1- \xi)}(1 - \Delta t_n \xi)(x - 1) \quad \textrm{with} \quad 1 < \xi < x.
\end{align*}
Either $h_n(x)$ falls below 1, or  
the contraction factor satisfies $e^{\Delta t_n(1- \xi)}(1 - \Delta t_n \xi) \leq 1 - \Delta t_{\rm min}$, since $\Delta t_n \geq \Delta t_{\rm min}$ from \cref{eq:Delta_t_min} and $\xi > 1$.
Then, after 
\begin{align}
\label{eq:N_2}
  N_2 = \frac{\log\frac{1}{\epsilon} + \log \bigl(\lambda_{\rm max}^{N_\theta} -1\bigr)}{\log \frac{1}{1 - \Delta t_{\rm min}}} = \mathcal{O}\bigl(\log\frac{1}{\epsilon}\bigr)
\end{align}
iterations, we obtain $\lambda_{n}^{N_\theta} \leq \epsilon + 1$. By \cref{eq:h_n-contraction-to-1},  during subsequent iterations, these eigenvalues greater than $1$ then remain in $(1 -\epsilon, 1+\epsilon)$ and continue to converge toward 1.

\textbf{Step 3. Warm-up phase for small eigenvalues.}
Next, we consider the smallest eigenvalue, $\lambda_{n+1}^{1} = \min_{1\leq i\leq N_\theta}\{h_n(\lambda_{n}^{i})\}$. We will show that after an additional warm-up phase of $N_3 = \mathcal{O}\Bigl(|\log \lambda_{\rm min}(\Sigma_0)|\Bigr)$ iterations, all eigenvalues become greater than $1 - \delta$ with $\delta = \frac{1}{3}$. 
After the previous phases ($N_1 + N_2$ iterations), we have 
$$ \lambda_{\rm min}(\Sigma_n) \geq \min\{\lambda_{\rm min}(\Sigma_0) , e^{-\beta}\},$$
since eigenvalues with $\lambda_n^i < 1$ continue to increase, while those with $\lambda_n^i > 1$ remain bounded below by $h_n(\lambda_{n}^{i}) = \lambda_n^i e^{\Delta t_n (1 -  \lambda_n^i)} \geq \lambda_n^i e^{-\beta} > e^{-\beta}$.
When the smallest eigenvalue satisfies  $\lambda_{n}^{1} < 1 - \delta$, we have
\begin{align*}
-\log \lambda_{n+1}^{1}  = \max_i\{-\log \lambda_n^{i} - \Delta t_n (1 - \lambda_n^{i})\} \leq -\log \lambda_{n}^{1} - \Delta t_{\rm  min} \delta. 
\end{align*}
Then after 
\begin{align}
\label{eq:N_3}
    N_3 = \frac{\max\{-\log \lambda_{\rm min}(\Sigma_{0}), \beta\}}{\Delta t_{\rm min} \delta} = \mathcal{O}\bigl(\bigl|\log \lambda_{\rm min}(\Sigma_{0})\bigr|\bigr)
\end{align}iterations, the smallest eigenvalue becomes larger than $1-\delta$.

\textbf{Step 4. Convergence phase for small eigenvalues.}
Finally, we consider the convergence phase for the smallest eigenvalue. After $N_4 = \mathcal{O}\bigl(\log\frac{1}{\epsilon}\bigr)$ iterations, it satisfies
$\lambda_{n}^{1} \geq 1 - \epsilon.$ Again,
$h_n(x)$ for $1-\delta < x < 1$ is a contraction:
\begin{align*}
h_n(1) - h_n(x) = h'_n(\xi) (1 - x) = e^{\Delta t_n(1- \xi)}(1 - \Delta t_n \xi)(1 - x) \quad \textrm{with} \quad x < \xi < 1.
\end{align*}
Using $\delta  = \frac{1}{3}$, the contraction rate can be bounded as
\begin{align*}
  e^{\Delta t_n(1- \xi)}(1 - \Delta t_n \xi) 
  &\leq e^{\Delta t_n\delta}\bigl(1 - \Delta t_n (1- \delta)\bigr) \\
  &\leq \frac{1}{1 - \Delta t_n\delta}\bigl(1 - \Delta t_n (1- \delta)\bigr)
  \leq 1 - \delta \Delta t_{n}
  \leq 1 - \delta \Delta t_{\rm min}.
\end{align*}
Here, the first inequality uses the fact that the expression is decreasing with respect to $\xi$; in the second inequality we use $e^{a} \leq \frac{1}{1 - a}$ for $a \in (0,1)$; the third inequality uses that $\delta = \frac{1}{3}$; and the last inequality uses the lower bound $\Delta t_n \geq \Delta t_{\rm min}$.
Therefore, after \begin{align}
\label{eq:N_4}
    N_4 = \frac{\log\frac{1}{\epsilon}}{\log \frac{1}{1 - \delta \Delta t_{\rm min}}} = \mathcal{O}\bigl( \log\frac{1}{\epsilon} \bigr)
\end{align}
iterations, we have $\lambda_{n}^{1} \geq 1 - \epsilon$. 

\textbf{Then we consider the convergence of $\{v_n\}$.} Since $\lVert v_{n+1} \rVert_2 \leq  \lVert I-\Delta t_n\Sigma_n \rVert_2  \lVert v_{n} \rVert_2$, it suffices to bound the spectral norm of $I-\Delta t_n\Sigma_n$.
Using the upper bound on $\Delta t_n$ in \cref{eq:adaptive-dt-gaussian} and the assumptions $\Delta t_{\rm max}, \beta \leq 1$, we obtain
\begin{align*}
     I \succeq I-\Delta t_n\Sigma_n = (1-\Delta t_n)I+\Delta t_n(I-\Sigma_n) \succeq (1-\Delta t_{\max}-\beta) I \succeq -I.
\end{align*}
Consequently, $\|v_n\|_2$ is non-increasing. If $\lVert v_0\rVert_2 \leq \epsilon$, the sequence has already converged. Otherwise, after the four steps described above, we have the refined bounds hold: $ (1-(1+\epsilon)\Delta t_{\max}) \preceq I-\Delta t_n\Sigma_n   \preceq (1-(1-\epsilon)\Delta t_{\min}) I$, where $\epsilon<1$, $\Delta t_{\rm max} < 1$, and $\Delta t_{\rm min} > 0$ is defined in \cref{eq:Delta_t_min}. 
Therefore after 
\begin{align}
\label{eq:N_5}
    N_5 = \frac{\log\frac{1}{\epsilon} + \max\{\log\|v_0\|_2,0\} }{\log \frac{1}{\max\{(1+\epsilon)\Delta t_{\max}-1, 1-(1-\epsilon)\Delta t_{\min}\}}} = \mathcal{O}\bigl(\log\frac{1}{\epsilon}+\max\{\log\|v_0\|_2,0\}\bigr)
\end{align}
iterations, we have $\|v_n\|_2\leq \epsilon$.

Combining all phases,  after $N_1 + N_2+N_3+N_4+N_5$ iterations as defined in \cref{eq:N_1,eq:N_2,eq:N_3,eq:N_4,eq:N_5}, all eigenvalues of $\Sigma_n$ lie in the interval $(1-\epsilon,1+\epsilon)$, and the norm of $v_n$ is below $\epsilon$, which completes the proof of \cref{eq:theorem-Sigma_n-no-noise}.
\end{proof}

\subsection{Proof of \texorpdfstring{\cref{thm:Sigma_n-noise}}{theorem 2}}
\label{proof_of_thm:Sigma_n-noise}

\begin{proof}[Proof of \cref{thm:Sigma_n-noise}]
\label{proof:Sigma_n-noise}
\textbf{Step 1. Derivation of the tail bounds of error terms.}
Recall from \cref{eq:f_t_n_theta} that \( f(t_n, \theta) = \frac{1}{2}\theta^T(\LL_n^T\LL_n-I)\theta+v_n^T\LL_n\theta+\text{const} \), which is a quadratic function of $\theta$.
Under \cref{assum:stoch_convergence}, the coefficients are uniformly bounded. We employ Monte Carlo sampling with \( \theta \sim \N(0,I)\) to estimate \( \E_\N[\theta\theta^T(f-\E_\N[f])] \) and \( \E_\N[\theta(f-\E_\N[f])] \). The corresponding Monte Carlo error terms are bounded by polynomial functions of $\lVert 
\theta \rVert_2$. Since the Gaussian tail bound $\mathbb{P}_{\theta \sim \N}[\lVert \theta \rVert_2^p \geq t] \lesssim e^{- t^\frac{2}{p}}$, it follows that there exist constants $C_1, C_2, \gamma>0$ such that 
\begin{equation}
\label{eq:P(Omega_n)}
   \mathbb{P}(\| \Omega_n\|_2 > t|\F_n) < C_1 e^{-C_2 t^\gamma}, \quad \mathbb{P}(\| u_n\|_2 > t|\F_n) < C_1 e^{-C_2 t^\gamma},
\end{equation}
where $\F_n$ is a filtration defined by $\F_n=\sigma(\{\Sigma_k,v_k\}_{1 \le k \le n})$. 

\textbf{Step 2. Probability estimate for the adaptive time-step selection.} Recall that the time step is $\Delta t_n= \min\{\eta(t_n) \Delta t_{\max},\frac{\beta}{\|-\LL_n^T \LL_n +I+\Omega_n\|_2}\}.$ Define 
$$k_n =  \frac{\beta}{\Delta t_{\max}\eta(t_n)}, A_n = \{ \Delta t_n = \frac{\beta}{\|-\LL_n^T \LL_n+I+\Omega_n\|_2} \} = \{ \|-\LL_n^T \LL_n+I+\Omega_n\|_2 \geq k_n \}.$$
Then 
\begin{align*}
    \mathbb{P}(A_n|\F_n) &\le \mathbb{P}(\|\Omega_n\|_2 \ge k_n - \|\LL_n^T \LL_n - I\|_2|\F_n) \\
    &\le \mathbb{P}(\|\LL_n^T \LL_n - I\|_2 > \frac{k_n}{2}|\F_n) + \E(1_{\{\|\LL_n^T \LL_n - I\|_2 \le \frac{k_n}{2}\}}1_{\{\|\Omega_n\|_2 \ge k_n - \|\LL_n^T \LL_n - I\|_2\}}|\F_n)\\
&\le \frac{\|\Sigma_n - I\|_2^2}{(k_n / 2)^2} + C_1 e^{- \frac{C_2}{2^\gamma} k_n^\gamma}, 
\end{align*}
where the last inequality follows from Markov's inequality and the tail bound in \cref{eq:P(Omega_n)}. By \cref{assum:stoch_convergence}, there exists $C_0 > 0$ such that $\E [\|\Sigma_n\|_2^2] < C_0^2$. Moreover, using the inequality $e^{-\frac{1}{x}} \lesssim x^{\frac{2}{\gamma}}$ when $x > 0$ to bound the exponential term, we obtain
\begin{align}
   \mathbb{P}(A_n|\F_n) &\le \frac{2\|\Sigma_n\|_2^2 + 2}{(k_n / 2)^2} + C_1 e^{-\frac{C_2 \beta^\gamma}{2^\gamma\Delta t_{\max}^\gamma\eta(t_n)^\gamma}}
    \notag\\
    &\leq \frac{8(C+1)\Delta t_{\max}^2}{\beta^2}\eta(t_n)^2 + C_3 \eta(t_n)^2 
     = \left(\frac{8(C+1)\Delta t_{\max}^2}{\beta^2} + C_3\right) \eta(t_n)^2, \label{eq:P(A_n)} 
\end{align}
for some constant $C_3 > 0$.
Since $\sum_{n=1}^{\infty} \eta(t_n)^2 < \infty$, it follows that $\sum_{n=1}^{\infty}\mathbb{P}(A_n) < \infty$. By Borel-Cantelli Lemma, we have $\mathbb{P}\left(A_n, \mathrm{i.o.} \right) = 0.$

\textbf{Step 3. $\Sigma_n$ converges to $I$ almost surely.}
 Define the Lyapunov function $V(\Sigma)={\rm tr}(\Sigma-\log\Sigma-I)$ . From the update rule, we obtain
\begin{align}
\label{eq:Sigma_n_update}
    V(\Sigma_{n+1})-V(\Sigma_n)=& {\rm tr}\big(\LL_n (e^{\Delta t_n(-\LL_n^T \LL_n+I+\Omega_n)}-I) \LL_n^T\big) \nonumber\\
    &- \Bigl({\rm tr} \big( \log ( \LL_n e^{\Delta t_n(-\LL_n^T \LL_n+I+\Omega_n)} \LL_n^T )\bigr) 
    - {\rm tr} (\log \Sigma_n)\Bigr).
\end{align}
Since $e^x-1\leq x + x^2 $ for $|x|\leq 1$ and $\lVert \Delta t_n(-\LL_n^T \LL_n+I+\Omega_n)\rVert_2 \leq \beta \leq 1$, we obtain
\begin{align}
     {\rm tr}\bigl(\LL_n (e^{\Delta t_n(-\LL_n^T \LL_n+I+\Omega_n)}-I) \LL_n^T\bigr)
     & \leq 
     {\rm tr}\Bigl(\LL_n \bigl(\Delta t_n (-\LL_n^T \LL_n+I+\Omega_n) + (\Delta t_n)^2(-\LL_n^T \LL_n+I+\Omega_n)^2 \bigr) \LL_n^T\Bigr) \nonumber
     \\
     & = 
     \Delta t_n {\rm tr}\bigl(\LL_n^T \LL_n(-\LL_n^T \LL_n+I+\Omega_n)\bigr) + (\Delta t_n)^2{\rm tr}\bigl(\LL_n^T \LL_n(-\LL_n^T \LL_n+I+\Omega_n)^2\bigr). \label{eq:Sigma_n_update_1}
\end{align}
Note that $\log \det \Sigma= {\rm tr} \log \Sigma$, we further obtain
\begin{align}
   {\rm tr} \bigl( \log ( \LL_n e^{\Delta t_n(-\LL_n^T \LL_n+I+\Omega_n)} \LL_n^T )\bigr)
    -  {\rm tr} (\log \Sigma_n) 
    &= 
    \log \det (\LL_n e^{\Delta t_n(-\LL_n^T \LL_n+I+\Omega_n)} \LL_n^T) -\log \det \Sigma_n \nonumber
    \\ &= 
    \Delta t_n {\rm tr}(-\LL_n^T \LL_n+I+\Omega_n). \label{eq:Sigma_n_update_2}
\end{align}
Substituting \cref{eq:Sigma_n_update_1,eq:Sigma_n_update_2} into \cref{eq:Sigma_n_update} and taking conditional expectations with respect to $\F_n$ yields, for $n \ge 1$, 
\begin{equation}
\label{eq:V_Sigma}
\begin{aligned}
\E[ V(\Sigma_{n+1})|\F_n ] \leq& V(\Sigma_{n}) + \E[  \Delta t_n {\rm tr}\bigl((\LL_n^T \LL_n-I)(-\LL_n^T \LL_n+I)\bigr) | \F_n]  +  \E[  \Delta t_n {\rm tr}\bigl((\LL_n^T \LL_n-I) \Omega_n\bigr) | \F_n] \\
&+ \E[ (\Delta t_n)^2{\rm tr}\bigl(\LL_n^T \LL_n(-\LL_n^T \LL_n+I+\Omega_n)^2\bigr)|\F_n]\\
=& V(\Sigma_{n}) - \Bigl(\E[  \Delta t_n \lVert \Sigma_n -I\rVert_F^2 | \F_n] - \min\{\E[  \Delta t_n {\rm tr}\bigl((\LL_n^T \LL_n-I)\Omega_n\bigr) | \F_n],0\} \Bigr)\\
&+ \underbrace{\E[ (\Delta t_n)^2{\rm tr}\bigl(\LL_n^T \LL_n(-\LL_n^T \LL_n+I+\Omega_n)^2\bigr)|\F_n]}_{Z^{(1)}_n} + \underbrace{\max\{\E[  \Delta t_n {\rm tr}\bigl((\LL_n^T \LL_n-I)\Omega_n\bigr) | \F_n], 0\}}_{Z^{(2)}_n}. 
\end{aligned}
\end{equation}
By the boundedness of $\LL_n$ from \cref{assum:stoch_convergence} and the tail bound in \cref{eq:P(Omega_n)}, there exists $C_4 > 0$ such that $\E\bigl[{\rm tr}\bigl(\LL_n^T \LL_n(-\LL_n^T \LL_n+I+\Omega_n)^2\bigr)| \F_n\bigr] < C_4$ and $\E\bigl[\bigl|{\rm tr}\bigl((\LL_n^T\LL_n -I) \Omega_n \bigr)\bigr|^2| \F_n\bigr] < C_4$ for all $n$. Since ${\rm tr}\bigl(\LL_n^T \LL_n(-\LL_n^T \LL_n+I+\Omega_n)^2\bigr) \geq 0$, we have
\begin{align*}
 \sum_{n=1}^\infty Z^{(1)}_n \leq \sum_{n=1}^\infty (\eta(t_n)\Delta t_{\max})^2 \E\big[ {\rm tr}\bigl(\LL_n^T \LL_n(-\LL_n^T \LL_n+I+\Omega_n)^2\bigr)|\F_n\big]\leq\sum_{n=1}^\infty (\eta(t_n)\Delta t_{\max})^2 C_4<\infty.
\end{align*}
From \cref{assum:stoch_convergence} the noise $\Omega_n$ has zero mean, i.e.,  $\E[\Omega_n|\F_n]=0$, and thus
\begin{align*}
    \sum_{n=1}^{\infty}Z^{(2)}_n
    &= \sum_{n=1}^{\infty}\max\{\E[  (\Delta t_n - \Delta t_{\rm max} \eta(t_n)) {\rm tr}\bigl((\LL_n^T \LL_n-I)\Omega_n\bigr) | \F_n], 0\} \\
    &=\sum_{n=1}^{\infty}\max\{\E[  (\Delta t_n - \Delta t_{\rm max} \eta(t_n))1_{A_n} {\rm tr}\bigl((\LL_n^T \LL_n-I)\Omega_n\bigr) | \F_n] , 0\}\\
    &\leq \sum_{n=1}^{\infty}2\Delta t_{\rm max} \eta(t_n) \sqrt{\E[1_{A_n}| \F_n]} \sqrt{\E[|{\rm tr}\bigl((\LL_n^T \LL_n-I)\Omega_n\bigr)|^2 | \F_n] } \\ 
    &\leq 2\Delta t_{\rm max}\sqrt{C_4}\sum_{n=1}^{\infty} \eta(t_n) \sqrt{\E[1_{A_n}| \F_n]}< \infty,
\end{align*}
where in the last inequality we used $\E[1_{A_n}| \F_n] = \mathbb{P}(A_n|\F_n) \lesssim \eta(t_n)^2$ from \cref{eq:P(A_n)}.
Applying the Robbins-Siegmund theorem \cite{robbins1971a} to \cref{eq:V_Sigma}, we have  $\{V(\Sigma_{n})\}_{n \geq 1}$ converges almost surely to a random variable $V_{\infty}$ and 
\begin{equation}
    \label{eq:Rob-Sigma}
    \sum_{n=1}^{\infty}\E[  \Delta t_n \lVert \Sigma_n -I\rVert_F^2 | \F_n] - \min\{\E[  \Delta t_n {\rm tr}\bigl((\LL_n^T \LL_n-I)\Omega_n\bigr) | \F_n],0\} < \infty
\end{equation}
almost surely.
Since  $\E[  \Delta t_n \lVert \Sigma_n -I\rVert_F^2 | \F_n] =
\E[  \Delta t_n 1_{A_n}\lVert \Sigma_n -I\rVert_F^2 | \F_n] + \E[  \Delta t_n 1_{A^C_n}\lVert \Sigma_n -I\rVert_F^2 | \F_n] $, we have almost surely that
$\sum_{n=1}^{\infty}\E[  \Delta t_n 1_{A^C_n}\lVert \Sigma_n -I\rVert_F^2 | \F_n]=
\sum_{n=1}^{\infty}  \Delta t_{\rm max} \eta(t_n) (1- \mathbb{P}(A_n)) \lVert \Sigma_n -I\rVert_F^2
< \infty.$
 Since $\sum_{n=1}^{\infty} \eta (t_n) = \infty$ and $\sum_{n=1}^{\infty} \mathbb{P}(A_n) < \infty$, it follows that $\liminf_{n\rightarrow \infty} \|\Sigma_n-I\|^2_F = 0$ almost surely. Combined with the almost sure convergence of $V(\Sigma_n)$, this implies $V(\Sigma_n) \xrightarrow{\text{a.s.}} 0$, and hence $\Sigma_n \xrightarrow{\text{a.s.}} I$.

\textbf{Step 4. $v_n$ converges to $0$ almost surely.} We compute
\[
\|v_{n+1}\|_2^2
= \|v_n - \Delta t_n \left( \LL_n(\LL_n^Tv_n + u_n) \right)\|_2^2
= \|v_n\|_2^2 - 2\Delta t_n  v_n^T \left( \LL_n(\LL_n^Tv_n + u_n) \right)  + \Delta t_n^2 \| \LL_n(\LL_n^Tv_n + u_n) \|_2^2.
\]
Taking conditional expectations with respect to the filtration $\F_n $ yields, for $n \geq 1$,
\begin{align}
\E\left[\|v_{n+1}\|_2^2 | \F_n\right]
\le& \|v_{n}\|_2^2 - \E[2\Delta t_n  (v_n^T \Sigma_n v_n + v_n^T\LL_n u_n)| \F_n] + \E [ \Delta t_n^2\| \LL_n(\LL_n^Tv_n + u_n) \|_2^2 \big| {\F}_n] \nonumber\\
=& \|v_{n}\|_2^2 - \bigl(\E[2\Delta t_n  v_n^T \Sigma_n v_n | \F_n] + \max\{\E[2\Delta t_n v_n^T\LL_n u_n| \F_n], 0\}\bigr) \label{eq:v_n+1}\\
& + \E [ \Delta t_n^2\| \LL_n(\LL_n^Tv_n + u_n) \|_2^2 \big| {\F}_n] - \min\{\E[2\Delta t_n v_n^T\LL_n u_n| \F_n], 0\}.\nonumber
\end{align}
By the boundedness of $\LL_n$ and $v_n$, from \cref{assum:stoch_convergence} and the tail bound in \cref{eq:P(Omega_n)}, there exists $C_5> 0$ such that 
$
  \E [ \Delta t_n^2\| \LL_n(\LL_n^Tv_n + u_n) \|_2^2 \big| {\F}_n] \le \Delta t_{\text{max}}^2 \eta(t_n)^2 \E\bigl[2 \lVert \LL_n \rVert_2^2 \bigl(\lVert \LL_n \rVert_2^2 \lVert v_n \rVert_2^2  +  \|u_n\|_2^2\bigr)\big| {\F}_n\bigr] \leq C_5\eta(t_n)^2,$ and $
  \E [ (v_n^T\LL_n u_n)^2\big| {\F}_n\bigr] \leq C_5.  
$
We have 
\begin{align*}
    \sum_{n=1}^{\infty} \E [ \Delta t_n^2\| \LL_n(\LL_n^Tv_n + u_n) \|_2^2 \big| {\F}_n] &\leq \sum_{n=1}^{\infty} C_5\eta(t_n)^2 < \infty,\\
   \sum_{n=1}^{\infty}-\min\{\E[2\Delta t_n v_n^T\LL_n u_n| \F_n], 0\} &=\sum_{n=1}^{\infty}-\min\{\E[2(\Delta t_n  - \Delta t_{\rm max}\eta(t_n))v_n^T\LL_n u_n| \F_n], 0\} \\
   &=\sum_{n=1}^{\infty}-\min\{\E[2(\Delta t_n  - \Delta t_{\rm max}\eta(t_n))1_{A_n}v_n^T\LL_n u_n| \F_n], 0\} \\
  &\leq \sum_{n=1}^{\infty}4\Delta t_{\rm max} \eta(t_n) \sqrt{\E[1_{A_n}| \F_n]} \sqrt{\E[  (v_n^T \LL_nu_n )^2 | \F_n] } \\ 
  &\leq \sum_{n=1}^{\infty}4\Delta t_{\rm max} \eta(t_n) \sqrt{\mathbb{P}(A_n|\F_n)}\sqrt{C_5} < \infty. 
\end{align*}
In the second equation, we used that $\E[u_n|\F_n]=0$ and $\mathbb{P}(A_n|\F_n) \lesssim \eta(t_n)^2$ in \cref{eq:P(A_n)}.
Applying the Robbins-Siegmund theorem \cite{robbins1971a} to \cref{eq:v_n+1}
 we have that $\{v_{n}\}_{n\geq 1}$ converges almost surely to a random variable $v_\infty$ and that 
 \begin{equation}
    \label{eq:Rob-v}
    \sum_{n=1}^{\infty} \big(\E[2\Delta t_n  v_n^T \Sigma_n v_n | \F_n] + \max\{\E[2\Delta  t_n v_n^T\LL_n u_n| \F_n], 0\}\big)< \infty
 \end{equation}  almost surely.
 Since $\E[2\Delta t_n  v_n^T \Sigma_n v_n | \F_n] \geq \E[2\Delta t_n 1_{A_n} v_n^T \Sigma_n v_n | \F_n]$, we have almost surely that 
 $\sum_{n=1}^{\infty} \E[2\Delta t_n  1_{A^C_n} v_n^T \Sigma_n v_n | \F_n] = \sum_{n=1}^{\infty} 2\Delta t_{\rm max}\eta(t_n) (1 - \mathbb{P}(A_n) )v_n^T \Sigma_n v_n  < \infty. $
Since $\sum_{n=1}^{\infty} \eta(t_n) = \infty$, $\sum_{n=1}^{\infty} \mathbb{P}(A_n) < \infty$, and  $\Sigma_n \xrightarrow{\text{a.s.}} I$, it follows that $\liminf_{n\rightarrow \infty} \|v_n\|^2_2 = 0$ almost surely. Combined with the almost sure convergence of  $\lVert v_n \rVert_2^2$, this implies $\lVert v_n \rVert_2^2 \xrightarrow{\text{a.s.}} 0$, and hence $v_n \xrightarrow{\text{a.s.}} 0$.
\end{proof}

\subsection{Proof of \texorpdfstring{\cref{theorem-Gaussian-Mixture-no-noise}}{theorem 3}}
\label{proof_of_theorem-Gaussian-Mixture-no-noise}

\begin{proof}[Proof of \cref{theorem-Gaussian-Mixture-no-noise}]

We divide the proof into two steps. The first step is the convergence of $\{ \Sigma_{n,k}, v_{n,k} \}$, and the second step is the convergence of $\{ w_{n,k} \}$.

\textbf{Step 1. Convergence phase for $\{ \Sigma_{n,k}, v_{n,k} \}$.}
We claim that, under the well-separation assumption, the evolution of $\Delta t_n$ and $\Sigma_{n,k}$, according to \cref{eq:adaptive-dt-gaussian-mixture} and \eqref{eq:Gaussian-mixture-matrix}b, can be rewritten as the Gaussian case in \cref{theorem-Sigma_n-no-noise} after integrating the covariance matrices into a block-diagonal system. To be precise, denote
\[
S_n = {\rm Diag}(\Sigma_{n,1},\Sigma_{n,2},\cdots,\Sigma_{n,K}),
\]
then
\[
S_{n+1} = {\rm Diag}(h_n(\Sigma_{n,1}), h_n(\Sigma_{n,2}),\cdots,h_n(\Sigma_{n,K})) = h_n(S_n),
\]
where $h_n(x):= xe^{\Delta t_n(1-x)}$, then
\[
\Delta t_n
=
\min\Bigl\{\Delta t_{\max}, \frac{\beta}{\max_k\{\|\Sigma_{n,k}-I\|_2 \}}\Bigr\}
=
\min\Bigl\{\Delta t_{\max}, \frac{\beta}{\|S_n-I\|_2}\Bigr\}.
\]

Since $\{S_n\}$ satisfies exactly the same recursion and adaptive time-stepping rule as in the Gaussian case, $\{\Sigma_n\}$ convergence part of \cref{theorem-Sigma_n-no-noise} applies directly to the block-diagonal system. 
Moreover, akin to $\{ v_n\}$ convergence part of \cref{theorem-Sigma_n-no-noise}, we have 
% $\lVert v_{n+1,k} \rVert_2 = \lVert I-\Delta t_n\Sigma_{n,k} \rVert_2  \lVert v_{n,k} \rVert_2$ and 
$\lVert v_{n,k} \rVert_2$ is non-increasing and converges to 0 exponentially fast.
Consequently, there exists
\[
N_{\Sigma,v}(\epsilon)
=
\mathcal O\left(
|\log \underline\lambda_0|
+
|\log \overline\lambda_0|
+
\max\{\log V_0,0\}
+
\log \frac{1}{\epsilon}
\right),
\]
with $\underline\lambda_0=\min_{1\le k\le K}\lambda_{\min}(\Sigma_{0,k})$,
$\overline\lambda_0=\max_{1\le k\le K}\lambda_{\max}(\Sigma_{0,k})$, and $
V_0=\max_{1\le k\le K}\|v_{0,k}\|_2$, 
such that for all $n\ge N_{\Sigma,v}(\epsilon)$,
\[
\max_{1\le k\le K}\|\Sigma_{n,k}-I\|_2\le \epsilon,
\quad
\max_{1\le k\le K}\|v_{n,k}\|_2\le \epsilon.
\]

For $1 \leq k \leq K$, define
\begin{equation*}
a_{n,k}
=
\frac12\Bigl(
\mathrm{tr}(\Sigma_{n,k}-I)+v_{n,k}^Tv_{n,k}-\log |\Sigma_{n,k}|
\Bigr),
\quad
r_{n,k}=\frac{w_{n,k}}{w_{\star,k}}.
\end{equation*}
Using
\[
\widehat w_{n+1,k}
=
w_{n,k}\Bigl(\frac{w_{\star,k}}{w_{n,k}}\Bigr)^{\Delta t_n} e^{-\Delta t_n a_{n,k}}
=
w_{\star,k} r_{n,k}^{1-\Delta t_n} e^{-\Delta t_n a_{n,k}},
\]
we have for any $1\le k,j\le K$,
\begin{equation}
\label{eq:weight-log-ratio-recursion}
\log \frac{r_{n+1,k}}{r_{n+1,j}}
=
(1-\Delta t_n)\log \frac{r_{n,k}}{r_{n,j}}
-
\Delta t_n (a_{n,k}-a_{n,j}),
\end{equation}
where the normalization cancels in pairwise ratios.

Define
\[
m_\star=\min\{\underline\lambda_0,e^{-\beta}\},
\quad
M_\star=\max\{1,\overline\lambda_0\},
\quad
\Phi(x) = x - 1 - \log x.
\]
From the proof of \cref{theorem-Sigma_n-no-noise}, all eigenvalues of all $\Sigma_{n,k}$ remain in the interval
$[m_\star,M_\star]$, and $\|v_{n,k}\|_2\le V_0$ for all $n$ and all $k$. Hence
\begin{equation}
\label{eq:a_n_k}
0\le a_{n,k}
=
\frac12\left(
\sum_{i=1}^{N_\theta}\Phi(\lambda_{n,k}^{i})+\|v_{n,k}\|_2^2
\right)
\le
A_\star:=\frac12\Bigl(
N_\theta \max\{\Phi(m_\star),\Phi(M_\star)\}+V_0^2
\Bigr),
\end{equation}
where $\{\lambda_{n,k}^{i}\}_{i=1}^{N_\theta}$ are eigenvalues of $\Sigma_{n,k}$. 
% \begin{equation}
% \label{eq:A_star}
% A_\star=\frac12\Bigl(
% N_\theta \max\{\Phi(m_\star),\Phi(M_\star)\}+V_0^2
% \Bigr) \leq \frac{1}{2}N_\theta \bigl(|\log \underline\lambda_0| + \overline\lambda_0 \bigr) + V_0^2.
% \end{equation}
Now let
\begin{equation}
Q_0=\max_{1\le k,j\le K}\left|\log \frac{r_{0,k}}{r_{0,j}}\right| \leq 2\max_{1\le k \le K}\left|\log  r_{0,k}\right|.
\end{equation}
By \eqref{eq:weight-log-ratio-recursion},
\[
\left|\log \frac{r_{n+1,k}}{r_{n+1,j}}\right|
\le
(1-\Delta t_n)\left|\log \frac{r_{n,k}}{r_{n,j}}\right|+\Delta t_n |a_{n,k}-a_{n,j}|
\le
(1-\Delta t_n)\left|\log \frac{r_{n,k}}{r_{n,j}}\right|+\Delta t_n A_\star.
\]
% Taking the maximum over $(k,j)$ yields
% \[
% Q_{n+1}\le (1-\Delta t_n)Q_n+\Delta t_n A_\star.
% \]
Therefore, by induction, for any $n\ge 0$,
\begin{equation}
\label{eq:Qn-uniform-bound}
\left|\log \frac{r_{n,k}}{r_{n,j}}\right|\le Q_\star:=\max\{Q_0,A_\star\}.
\end{equation}
Furthermore, $Q_\star$ satisfies
\begin{equation}
\label{eq:logQ_star}
\begin{split}
    \log Q_\star
    &\leq \max\{\log Q_0, \log A_\star, 0\} \\
    &\leq \max\left\{\log N_\theta + \max\{\log \Phi(m_\star), \log \Phi(M_\star)\} , \log V_0^2 ,  \log\left( 2\max_{1 \le k \le K} \left| \log r_{0,k} \right|\right), 0 \right\}\\
    &\leq \max\left\{\log N_\theta + \max\{\beta, |\log \underline\lambda_0|, |\log \overline\lambda_0|\} , \log V_0^2 ,  \log \left( 2\max_{1 \le k \le K} \left| \log r_{0,k}\right|\right), 0 \right\},
\end{split}
\end{equation}
where in the last inequality we used
\begin{align*}
    &\log \Phi(m_\star) 
    \le \max \left\{ \beta, 1_{\{ \underline\lambda_0 < e^{-\beta} \}}\log |\log {\underline\lambda_0}| \right\} 
    \le \max \left\{ \beta, |\log {\underline\lambda_0}| \right\},
    \\
    &\log \Phi(M_\star) 
    \le \max \{ 0, 1_{\{ \overline\lambda_0 > 1 \}} \log \overline\lambda_0 \}
    \le  |\log \overline\lambda_0|.
\end{align*}
This indicates that the weights may expand during this phase, but only within a fixed finite multiplicative range.

\textbf{Step 2. Convergence phase of the weights $\{w_{n,k}\}$.}
Applying step 1 of the proof with tolerance $\frac{\epsilon}{N_\theta + 1}$, 
there exists
\begin{equation}
\label{eq:N_Sigma_v}
   N_{\Sigma,v}
=
\mathcal{O}\Bigl(
|\log \underline\lambda_0|
+
|\log \overline\lambda_0|
+
\max\{\log V_0,0\}
+
\log \frac{N_\theta + 1}{\epsilon}\Bigr) 
\end{equation}
such that for all $n\ge N_{\Sigma,v}$,
\[
\max_{1\le k\le K}\|\Sigma_{n,k}-I\|_2\le \frac{\epsilon}{N_\theta + 1},
\quad
\max_{1\le k\le K}\|v_{n,k}\|_2\le \frac{\epsilon}{N_\theta + 1}.
\]
If $\epsilon<\beta/\Delta t_{\max}$, then the adaptive time-stepping yields $\Delta t_n=\Delta t_{\max}$
for all $n\ge N_{\Sigma,v}$.
Moreover, for $\epsilon < \frac{1}{2}$, we have $\Phi(x) \le (x-1)^2$ for all $|x-1| \le \epsilon$. Thus for all $n > N_{\Sigma,v}$, using \cref{eq:a_n_k} leads to 
\[
0 \le a_{n,k} \le \frac{1}{2} \left( N_\theta +1 \right) \left(\frac{\epsilon}{N_\theta + 1}\right)^2 \le \frac{1}{2} \epsilon.
\]
By \cref{eq:weight-log-ratio-recursion}, we have
\[
\left|\log \frac{r_{n+1,k}}{r_{n+1,j}}\right|
\le
(1-\Delta t_{\max})\left|\log \frac{r_{n,k}}{r_{n,j}}\right|+\Delta t_{\max} |a_{n,k}-a_{n,j}|
\le
(1-\Delta t_{\max})\left|\log \frac{r_{n,k}}{r_{n,j}}\right|+\frac{1}{2} \Delta t_{\max} \epsilon,
\]
which leads to 
\[
\left|\log \frac{r_{n+1,k}}{r_{n+1,j}}\right| - \frac{1}{2} \epsilon \le (1-\Delta t_{\max})\left(\left|\log \frac{r_{n,k}}{r_{n,j}}\right| - \frac{1}{2} \epsilon\right).
\]
For any $n > N_{\Sigma,v}$, iterating this recursion gives
\[
\left|\log \frac{r_{n,k}}{r_{n,j}}\right| - \frac{1}{2} \epsilon \le (1-\Delta t_{\max})^{n-N_{\Sigma,v}} \left( \left|\log \frac{r_{N_{\Sigma,v},k}}{r_{N_{\Sigma,v},j}}\right| - \frac{1}{2} \epsilon \right)
\le 
(1-\Delta t_{\max})^{n-N_{\Sigma,v}} \max\{Q_\star,\frac{1}{2}\epsilon\},
\]
where the last inequality used \eqref{eq:Qn-uniform-bound}.
Therefore, combing \cref{eq:logQ_star,eq:logQ_star,eq:N_Sigma_v} gives that after
\begin{equation}
\begin{split}
    N_{\Sigma,v} &+ \frac{\log \frac{2}{\epsilon} + \max\{\log Q_\star,0\}}{\log \frac{1}{1-\Delta t_{\max}}} 
    \\&= \mathcal{O}\Bigl(|\log \underline\lambda_0| + |\log \overline\lambda_0| + \max\{\log V_0, \log \max_{1 \le k \le K} \left| \log r_{0,k} \right| , 0\} + \log \frac{1}{\epsilon} + \log N_\theta \Bigr)
\end{split}
\end{equation}
iterations, we have 
$\left|\log \frac{r_{n,k}}{r_{n,j}}\right| \le \epsilon$ for all $1 \le k \le K$, where the hidden constant in the $\mathcal{O}(\cdot)$ depends only on $\Delta t_{\max}$ and $\beta$.
Since
$
\sum_{k=1}^K w_{\star,k} r_{n,k}
=
\sum_{k=1}^K w_{n,k}
=
1,
$
there exist indices $k_-(n)$ and $k_+(n)$ such that
$
r_{n,k_-(n)}\le 1\le r_{n,k_+(n)}.
$
Then for any $k$,
\[
\log r_{n,k} \le \log r_{n,k} - \log r_{n,k_-(n)}  \le \epsilon, \quad \log r_{n,k} \ge \log r_{n,k} - \log r_{n,k_+(n)}  \ge -\epsilon.
\]
Therefore, we have
$| \log w_{n,k} - \log w_{\star,k} | \le \epsilon$ for all $1 \le k \le K$.
\end{proof}

\subsection{Proof of \texorpdfstring{\cref{thm:GM-noise}}{theorem 4}}
\label{Proof_of_thm:GM-noise}
\begin{proof}[Proof of \cref{thm:GM-noise}]

\textbf{Step 1. Probability estimate for the adaptive time-step selection and convergence of $\Sigma_{n,k}, v_{n,k}$.}
Define the block-diagonal matrices and vectors
\[
S_n
=
\mathrm{Diag}(\Sigma_{n,1},\ldots,\Sigma_{n,K}),
\quad
V_n
=
(v_{n,1}^T,\ldots,v_{n,K}^T)^T,
\quad
U_n
=
(u_{n,1}^T,\ldots,u_{n,K}^T)^T,
\]
\[
\mathcal L_n
=
\mathrm{Diag}(\LL_{n,1},\ldots,\LL_{n,K}),
\quad
\Omega_n
=
\mathrm{Diag}(\Omega_{n,1},\ldots,\Omega_{n,K}).
\]
Then
$
\mathcal L_n\mathcal L_n^T
=
S_n .
$
The covariance and mean updates in \cref{eq:Gaussian-mixture-stoch} can be written as
\[
S_{n+1}
=
\mathcal L_n
\exp\left\{
\Delta t_n
\left(
-\mathcal L_n^T\mathcal L_n+I+\Omega_n
\right)
\right\}
\mathcal L_n^T,
\]
and
\[
V_{n+1}
=
V_n
-
\Delta t_n
\mathcal L_n
\left(
\mathcal L_n^T V_n+U_n
\right).
\]
Since all matrices are block diagonal, we have
\[
\left\|
-\mathcal L_n^T\mathcal L_n+I+\Omega_n
\right\|_2
=
\max_{1\le k\le K}
\left\|
-\LL_{n,k}^T\LL_{n,k}+I+\Omega_{n,k}
\right\|_2 .
\]
Therefore, the time step is exactly the stochastic adaptive time step for the block variable $S_n$:
\[
\Delta t_n
=
\min\left\{
\eta(t_n)\Delta t_{\max},
\frac{\beta}
{\left\|
-\mathcal L_n^T\mathcal L_n+I+\Omega_n
\right\|_2}
\right\}.
\]
Hence, the pair $(S_n,V_n)$ satisfies the same stochastic recursion as in the single Gaussian case in dimension $KN_\theta$.

Let
\[
A_n
=
\left\{
\Delta t_n
=
\frac{\beta}
{\left\|
-\mathcal L_n^T\mathcal L_n+I+\Omega_n
\right\|_2}
\right\}.
\]
As in the proof of \cref{thm:Sigma_n-noise}, using the conditional tail bound of the noise $\Omega_n$ and the boundedness assumption, we obtain
$
\mathbb P(A_n| \F_n)
\le
C\eta(t_n)^2
$
for some deterministic constant $C>0$, where $\F_n$ is a filtration defined by $\F_n=\sigma(\{\Sigma_{m,k},v_{m,k},w_{m,k}\}_{1 \le m \le n,1 \le k \le K})$. Since
$
\sum_{n=1}^\infty \eta(t_n)^2<+\infty,
$
it follows that
$
\sum_{n=1}^\infty \mathbb P(A_n)<+\infty .
$
By Borel-Cantelli lemma,
$
\mathbb P(A_n,\mathrm{i.o.})=0 .
$

By the block-diagonal formulation above, \cref{thm:Sigma_n-noise} applies directly to $(S_n,V_n)$. Hence
$
S_n\xrightarrow{\mathrm{a.s.}} I_{KN_\theta}
$ and $
V_n\xrightarrow{\mathrm{a.s.}}0 .
$
Equivalently, for every $1\le k\le K$,
$
\Sigma_{n,k}\xrightarrow{\mathrm{a.s.}} I
$ and $
v_{n,k}\xrightarrow{\mathrm{a.s.}}0.
$

\textbf{Step 2. $\log r_{n,k}$ converges to $0$ almost surely.}
For $1 \leq k \leq K$, define
\[
a_{n,k}
=
\frac12\Bigl(
\mathrm{tr}(\Sigma_{n,k}-I)+v_{n,k}^Tv_{n,k}-\log |\Sigma_{n,k}|
\Bigr),
\]
Using
\[
\widehat w_{n+1,k}
=
w_{n,k}\Bigl(\frac{w_{\star,k}}{w_{n,k}}\Bigr)^{\Delta t_n} e^{\Delta t_n (-a_{n,k} + s_{n,k})}
=
w_{\star,k} r_{n,k}^{1-\Delta t_n} e^{\Delta t_n (-a_{n,k} + s_{n,k})},
\]
for any $1\le k,j\le K$, we have
\begin{equation}
\label{eq:random-weight-log-ratio-recursion}
\log \frac{r_{n+1,k}}{r_{n+1,j}}
=
(1-\Delta t_n)\log \frac{r_{n,k}}{r_{n,j}}
-
\Delta t_n (a_{n,k}-a_{n,j})
+ 
\Delta t_n (s_{n,k}-s_{n,j}).
\end{equation}
Taking the expectation leads to 
\begin{equation}
\label{eq:w_n+1}
\begin{split}
\E\left[\left|\log \frac{r_{n+1,k}}{r_{n+1,j}}\right|^2 | \F_n\right]
\le& \E\left[ (1-\Delta t_n )^2 | \F_n\right]\left|\log \frac{r_{n,k}}{r_{n,j}}\right|^2  +   
\left(\Delta t_{\rm max}\eta(t_n)\right)^2(a_{n,k}-a_{n,j})^2 \\
&+ \left(\Delta t_{\rm max}\eta(t_n)\right)^2 \E\left[ (s_{n,k}-s_{n,j})^2| \F_n\right]
+ Z_{n}^{(1)} + Z_{n}^{(2)}\\
\leq& \left|\log \frac{r_{n,k}}{r_{n,j}}\right|^2 - \E\left[ 2\Delta t_n  - \Delta t_n^2 | \F_n\right]\left|\log \frac{r_{n,k}}{r_{n,j}}\right|^2 \\
&+ 
\left(\Delta t_{\rm max}\eta(t_n)\right)^2(a_{n,k}-a_{n,j})^2 \\
&+ \left(\Delta t_{\rm max}\eta(t_n)\right)^2 \E\left[ (s_{n,k}-s_{n,j})^2| \F_n\right]
+ |Z_{n}^{(1)}| + |Z_{n}^{(2)}|
,
\end{split}
\end{equation}
where 
\begin{align*}
Z_{n}^{(1)} &= \E\left[2\Delta t_n  \bigl( (1-\Delta t_n) \log \frac{r_{n,k}}{r_{n,j}} - \Delta t_n (a_{n,k}-a_{n,j}) \bigr) (s_{n,k}-s_{n,j})  | \F_n\right],\\
Z_{n}^{(2)} &= -\E\left[2(1-\Delta t_n)\Delta t_n \log \frac{r_{n,k}}{r_{n,j}} (a_{n,k}-a_{n,j}) | \F_n\right]
.
\end{align*}

By \cref{assum:GM-stoch_convergence}, there exist constants $A_\star, C_6$ such that
\begin{equation}
\label{assum:GM-stoch_convergence_2}
a_{n,k} \le A_\star, \quad \E\left[ (s_{n,k}-s_{n,j})^2 | \F_n \right] \le C_6,
\end{equation}
where the latter inequality is similar to the derivation in \cref{thm:Sigma_n-noise} and follows from Monte Carlo noise assumption of $s_{n,k}$ and the boundedness of $w_{n,k}$, and moment estimation is used.
Then we have
\begin{align*}
\left(\Delta t_{\rm max}\eta(t_n)\right)^2(a_{n,k}-a_{n,j})^2 \le \left(\Delta t_{\rm max}\eta(t_n)\right)^2 A_{\star}^2, \\
\left(\Delta t_{\rm max}\eta(t_n)\right)^2 \E\left[ (s_{n,k}-s_{n,j})^2| \F_n\right] \le \left(\Delta t_{\rm max}\eta(t_n)\right)^2 C_6, 
\end{align*}
which are both summable.

Next, we estimate $Z_{n}^{(1)}$ and $Z_{n}^{(2)}$. 
% Denote 
% \[
% Q_n=\max_{1\le k,j\le K}\left|\log \frac{r_{n,k}}{r_{n,j}}\right|,
% \text{ for all } n \geq 0.
% \]
By assumption, $\left|\log \frac{r_{n,k}}{r_{n,j}}\right| \le 2 C_0$, then we have
\begin{align*}
Z_{n}^{(1)} =& 2\log \frac{r_{n,k}}{r_{n,j}} \E[\left(\Delta t_n (1-\Delta t_n) - \Delta t_{\rm max}\eta(t_n) (1 - \Delta t_{\rm max}\eta(t_n)) \right)  (s_{n,k}-s_{n,j})  | \F_n]\\
&- 2(a_{n,k}-a_{n,j})\E[(\Delta t_n^2 - (\Delta t_{\rm max}\eta(t_n))^2)  (s_{n,k}-s_{n,j})  | \F_n]\\
=& 2\log \frac{r_{n,k}}{r_{n,j}} \E[\left(\Delta t_n (1-\Delta t_n) - \Delta t_{\rm max}\eta(t_n) (1 - \Delta t_{\rm max}\eta(t_n)) \right) 1_{A_n} (s_{n,k}-s_{n,j})  | \F_n]\\
&- 2(a_{n,k}-a_{n,j})\E[(\Delta t_n^2 - (\Delta t_{\rm max}\eta(t_n))^2) 1_{A_n} (s_{n,k}-s_{n,j})  | \F_n].
\end{align*}
Using \cref{assum:GM-stoch_convergence_2} and Cauchy's inequality leads to 
\begin{align*}
|Z_{n}^{(1)}| \le& 2 (2C_0 + A_{\star})\Delta t_{\rm max}\eta(t_n) \sqrt{\E[(s_{n,k}-s_{n,j})^2|\F_n]} \sqrt{\mathbb{P}(A_n|\F_n)}\\
\le & 2 \sqrt{C_6} (2C_0 + A_{\star})\Delta t_{\rm max}\eta(t_n) \sqrt{\mathbb{P}(A_n|\F_n)}.
\end{align*}
Since $\mathbb{P}(A_n|\F_n) \lesssim \eta(t_n)^2$, we have
$
\sum_{n=1}^\infty |Z_n^{(1)}| < \infty.
$

For $Z_{n}^{(2)}$, we have $
\sum_{n=1}^\infty |Z_n^{(2)}| \le \sum_{n=1}^\infty 4 C_0 \E[\Delta t_n (a_{n,k} + a_{n,j})|\F_n].
$
Using step 1 and \cref{eq:Rob-Sigma,eq:Rob-v} in proof of \cref{thm:Sigma_n-noise}, we have $ \sum_{n=1}^\infty \E[\Delta t_n \| \Sigma_{n,k} - I \|_F^2|\F_n] < \infty $ and $ \sum_{n=1}^\infty \E [ \Delta t_n v_{n,k}^T \Sigma_{n,k} v_{n,k}|\F_n ] < \infty$ almost surely. 
Thus, we can obtain 
\begin{align*}
    \sum_{n=1}^\infty  \E[\Delta t_n a_{n,k} |\F_n] \leq \sum_{n=1}^\infty 2 \E \left[ \Delta t_n \bigl(\frac{1}{2}\| \Sigma_{n,k} - I \|_F^2 + v_{n,k}^T  v_{n,k}\bigr)|\F_n \right] < \infty
\end{align*}
Here we used that $\Sigma_{n,k}$ converges to $I$ almost surely and $x-1-\log x \le (x-1)^2$ for all $|x-1| \le \frac{1}{2}$, 
This means that almost surely we have
$
\sum_{n=1}^\infty |Z_n^{(2)}| \le \sum_{n=1}^\infty 4 C_0 \E[\Delta t_n (a_{n,k} + a_{n,j})|\F_n]) < \infty.
$

% For $Z_{n}^{(2)}$, by step 1 and \cref{eq:Rob-Sigma,eq:Rob-v} in proof of \cref{thm:Sigma_n-noise}, we have $ \sum_{n=1}^\infty \E[\Delta t_n \| \Sigma_{n,k} - I \|_F^2|\F_n] < \infty $ and $ \sum_{n=1}^\infty \E [ \Delta t_n v_{n,k}^T \Sigma_{n,k} v_{n,k}|\F_n ] < \infty$ almost surely. 
% % Thus we can obtain 
% % \[\sum_{n=1}^\infty \E[\Delta t_{\max} \eta(t_n) \| \Sigma_{n,k} - I \|_F^2 1_{A_n^C}|\F_n] = \sum_{n=1}^\infty \Delta t_{\max} \eta(t_n) \| \Sigma_{n,k} - I \|_F^2 (1 - \mathbb{P}(A_n|\F_n)) < \infty,\]  
% % \[\sum_{n=1}^\infty \E [ \Delta t_{\max} \eta(t_n) v_{n,k}^T \Sigma_{n,k} v_{n,k}1_{A_n^C}|\F_n ] = \sum_{n=1}^\infty \Delta t_{\max} \eta(t_n) v_{n,k}^T \Sigma_{n,k} v_{n,k} (1 - \mathbb{P}(A_n|\F_n)) < \infty.\] 
% By Step 1, $\Sigma_{n,k}$ converges to $I$ almost surely and $x-1-\log x \le (x-1)^2$ for all $|x-1| \le \frac{1}{2}$, 
% we have almost surely that
% $
% \sum_{n=1}^\infty \E[\Delta t_n a_{n,k}|\F_n] < \infty,
% $
% which means that almost surely we have
% $
% \sum_{n=1}^\infty |Z_n^{(2)}| \le \sum_{n=1}^\infty 4 C_0 \E[\Delta t_n (a_{n,k} + a_{n,j})|\F_n] < \infty.
% $

Applying the Robbins-Siegmund theorem \cite{robbins1971a} to \cref{eq:w_n+1}
we have that $\left\{ \left|\log \frac{r_{n,k}}{r_{n,j}}\right|^2  \right\}_{n \ge 1}$ converges and $\sum_{n=1}^\infty \E[2 \Delta t_n - \Delta t_n^2 | \F_n] \left|\log \frac{r_{n,k}}{r_{n,j}}\right|^2 < \infty$ almost surely. 
Since $$\sum_{n=1}^\infty \E[2 \Delta t_n - \Delta t_n^2 | \F_n] \ge \sum_{n=1}^\infty \E[(2 \Delta t_n - \Delta t_n^2) 1_{A_n^C} | \F_n],$$ 
we have almost surely that \begin{align*}
    &\sum_{n=1}^\infty \E\left[(2 \Delta t_n - \Delta t_n^2) 1_{A_n^C} \left|\log \frac{r_{n,k}}{r_{n,j}}\right|^2 | \F_n\right] \\= &\sum_{n=1}^\infty \E\left[(2 \Delta t_{\rm max}\eta(t_n) - (\Delta t_{\rm max}\eta(t_n))^2) \left|\log \frac{r_{n,k}}{r_{n,j}}\right|^2 (1 - \mathbb{P}(A_n| \F_n))\right] < \infty.
\end{align*}
Since $\sum_{n=1}^\infty \eta(t_n)^2 < \infty$, $\mathbb{P}(A_n|\F_n) \lesssim \eta(t_n)^2$ and $\left|\log \frac{r_{n,k}}{r_{n,j}}\right|^2$ converges, we have \\ $\sum_{n=1}^\infty \E\left[(2 \Delta t_{\rm max}\eta(t_n) \left|\log \frac{r_{n,k}}{r_{n,j}}\right|^2 | \F_n\right] < \infty$, which further implies $\left|\log \frac{r_{n,k}}{r_{n,j}}\right|^2$ converges to $0$ almost surely. Note that $\max_{1 \le k \le K} \left| \log r_{n,k} \right| \le \max_{1 \le k,j  \le K} \left| \log \frac{r_{n,k}}{r_{n,j}} \right|$, we have $\log r_{n,k}$ converges to $0$ almost surely.

\end{proof}

\section{Proof of the Manifold Optimization}
\label{Proof:manifold}
% First of all, we explain that the expressions of Riemannian metric do not depend on the square-root factor. To be precise, we have the following lemma. 
\begin{lemma}
\label{lemma:SPD-manifold-square-root}
    For every square-root factorization $X = LL^T$, we have
    \begin{align*}
        {\rm d}(X,Y) &= \lVert \log(L^{-1}YL^{-T}) \rVert_F,\\
        {\rm exp}_X(\sigma) &= L\exp(L^{-1}\sigma L^{-T}) L^{T},\\
        {\rm exp}_{X}^{-1}(Y) &= L\log(L^{-1}Y L^{-T}) L^{T}.
    \end{align*}
    \label{lemma:square-root}
\end{lemma}
\begin{proof}
    As $(X^{\frac{1}{2}}L^{-T})(X^{\frac{1}{2}}L^{-T})^T = X^{\frac{1}{2}}L^{-T}L^{-1}X^{\frac{1}{2}} =  I$, $X^{\frac{1}{2}}L^{-T}$ is orthogonal. 
    Using this property and the identity $X^{\frac{1}{2}}L^{-T} = (L^{-1}X^{\frac{1}{2}})^{-1} = X^{-\frac{1}{2}}L$, we have
    \[
    \begin{aligned}
    \lVert \log(X^{-\frac{1}{2}}YX^{-\frac{1}{2}}) \rVert_F &= \lVert (X^{\frac{1}{2}}L^{-T})^T\log(X^{-\frac{1}{2}}YX^{-\frac{1}{2}}) (X^{\frac{1}{2}}L^{-T}) \rVert_F \\
    &= \left\lVert \log \left( (X^{\frac{1}{2}}L^{-T})^TX^{-\frac{1}{2}}YX^{-\frac{1}{2}} (X^{\frac{1}{2}}L^{-T})\right) \right\rVert_F\\
    &= \left\lVert \log ( L^{-1}YL^{-T} ) \right\rVert_F,\\
    X^{\frac{1}{2}}\exp(X^{-\frac{1}{2}}\sigma X^{-\frac{1}{2}}) X^{\frac{1}{2}} &= X^{\frac{1}{2}}\exp(X^{-\frac{1}{2}}LL^{-1}\sigma L^{-T} L^T X^{-\frac{1}{2}}) X^{\frac{1}{2}}\\
    &= X^{\frac{1}{2}}X^{-\frac{1}{2}}L\exp(L^{-1}\sigma L^{-T} ) L^T X^{-\frac{1}{2}} X^{\frac{1}{2}}\\
    &= L\exp(L^{-1}\sigma L^{-T}) L^{T},\\
    X^{\frac{1}{2}}\log(X^{-\frac{1}{2}}Y X^{-\frac{1}{2}}) X^{\frac{1}{2}} &= X^{\frac{1}{2}}\log(X^{-\frac{1}{2}}LL^{-1}Y L^{-T} L^T X^{-\frac{1}{2}}) X^{\frac{1}{2}}\\
    &= X^{\frac{1}{2}}X^{-\frac{1}{2}}L\log(L^{-1}Y L^{-T} ) L^T X^{-\frac{1}{2}} X^{\frac{1}{2}}\\
    &= L\log(L^{-1}Y L^{-T}) L^{T}.\\
    \end{aligned}
    \]
\end{proof}

\begin{proof}[Proof of \cref{theorem-manifold}]

We first prove the mirror descent \cref{eq:mirror_descent} is consistent with our update \cref{eq:gm-ngf-update_C,eq:gm-ngf-update_mw}. By the construction of our algorithm, we only need to prove
\begin{align}
    \label{eq:covariance_update_g}
    C_{n+1,k} &= \exp_{C_{n,k}}\Big( -\frac{2\Delta t_n}{w_{n,k}} g_{C_{n,k}} \Big),\\
    \label{eq:mean_update_g}
    m_{n+1,k} &= m_{n,k} - \Delta t_n\frac{C_{n,k}}{w_{n,k}} g_{m_{n,k}},\\
    \label{eq:weight_update_g}
    w_{n+1,k} &= \frac{w_{n,k}\exp(-\Delta t_n g_{w_{n,k}})}{\sum_{j=1}^K w_{j,n}\exp(-\Delta t_n g_{w_{j,n}})},
\end{align}
where $g_{C_{n,k}} = {\rm grad}_{C_{k}}{\rm KL}[\rho_a^{\rm{GM}} \Vert \pi]|_{C_{n,k}}$, $g_{m_{n,k}} = \nabla_{m_k}{\rm KL}[\rho_a^{\rm{GM}} \Vert \pi]|_{m_{n,k}}$, $g_{w_{n,k}} = \nabla_{w_k}{\rm KL}[\rho_a^{\rm{GM}} \Vert \pi]|_{w_{n,k}}$.

Note that the mirror descent \cref{eq:mirror_descent} can decompose as a sum over \(k\) components involving the partial gradients with respect to \(m_k,C_k,w_k\), and each component can be divided into covariance,   mean, and weight terms, which can be analyzed separately.

For the covariance term, let $f(C_{k}) = {\rm tr}(S_{n,k}\log S_{n,k} - S_{n,k})$, $S_{n,k} = C_{n,k}^{-\frac{1}{2}} C_k  C_{n,k}^{-\frac{1}{2}}$. Then $\nabla f(C_k) = C_{n,k}^{-1} \exp^{-1}_{C_{n,k}}(C_k) C_{n,k}^{-1}$ and $\nabla f(C_k)|_{C_{n,k}} = 0$. Therefore, we have 
\[
C_{n+1,k} = \argmin_{C_k}\Bigl\{ \langle C_{n,k}^{-1} g_{C_{n,k}} C_{n,k}^{-1}, C_k-C_{n,k} \rangle + \frac{w_{n,k}}{2 \Delta t_n} f(C_{k}) \Bigr\},
\]
 By the first-order optimality condition,
\[
C_{n,k}^{-1} g_{C_{n,k}} C_{n,k}^{-1} + \frac{w_{n,k}}{2 \Delta t_n} C_{n,k}^{-1} \exp^{-1}_{C_{n,k}}(C_{n+1,k}) C_{n,k}^{-1} = 0,
\]
which means $C_{n+1,k} = \exp_{C_{n,k}}\left( -\frac{2\Delta t_n}{w_{n,k}} g_{C_{n,k}} \right)$.
For the mean term,
\[
m_{n+1,k} = \argmin_{m}\Bigl\{ \langle m - m_{n,k}, g_{m_{n,k}} \rangle + \frac{w_{n,k}}{2 \Delta t_n} ( m - m_{n,k} )^T {{C}_{n,k}^{-1}} ( m - m_{n,k} ) \Bigr\},
\]
this directly yields \cref{eq:mean_update_g}.
For the weight term, define ${\rm KL}(w, w_n) = \sum_{k=1}^K w_k(\log w_k - \log w_{n,k})$, we have
\begin{align*}
w_{n+1} = \argmin_{w}\Bigl\{ \langle w - w_{n}, [g_{w_{n,1}},\cdots,g_{w_{n,K}} ]^T\rangle + \frac{1}{\Delta t_n} {\rm KL}(w, w_n) \Bigr\}  
     \quad {\rm s.t.}\quad w^T\mathbf{1} = 1,
\end{align*}
which directly yields \cref{eq:weight_update_g}.

% In particular, for the exponential integrator \cref{eq:gm-ngf-update_C}\cref{eq:gm-ngf-update_mw}, we can write the update in the form of proximal gradient descent as follows:
% \begin{equation}
% \label{eq:proximal}
% \begin{aligned}
%     C_{n+1,k} =& \argmin_{C}\Bigl\{ g_{C_{n,k}}( {\rm exp}_{C_{n,k}}^{-1}(C), {\rm grad}_{C_{n,k}} {f} ) + \frac{1}{2 \Delta t} \frac{w_{n,k}}{2} {\rm d}(C, C_{n,k})^2  \Bigr\}, \\
%     % m_{n+1,k}=& \argmin_{C}\Bigl\{ g_{C_{n,k}}( {\rm exp}_{C_{n,k}}^{-1}(C), \frac{2}{w_k}{\rm grad} {f}(C_{n,k}) ) + \frac{1}{2 \Delta t} { g_{C_{n,k}}}({\rm exp}_{C_{n,k}}^{-1}(C),{\rm exp}_{C_{n,k}}^{-1}(C))\Bigr\} \\
%     % m_{n+1,k}
%     % =& \argmin_{m}\Bigl\{ \langle m - m_{n,k}, \frac{C_{n,k}}{w_k}\nabla {f} (m_{n,k}) \rangle + \frac{1}{2 \Delta t} \lVert m - m_{n,k} \rVert_{}^2  \Bigr\} \\
%     = & \argmin_{m}\Bigl\{ \langle m - m_{n,k}, \nabla_{m_{n,k}}  {f}\rangle + \frac{w_{n,k}}{2 \Delta t} \lVert m - m_{n,k} \rVert_{{C}_{n,k}^{-1}}^2  \Bigr\}, \\
%      w_{n+1} = & \argmin_{w}\Bigl\{ \langle w - w_{n}, \nabla_{w_{n}} {f} \rangle + \frac{1}{\Delta t} {\rm KL}(w, w_n) \Bigr\} \\
%      &\quad {\rm s.t.}\quad w^T\mathbf{1} = 1,
% \end{aligned}
% \end{equation}
% where $f$ is ${\rm KL}[\rho_a^{\rm{GM}} \Vert \pi]$, and ${\rm KL}(w, w') = \sum_{k=1}^{K} w_k (\log w_k - \log w_k')$.

Next, we prove the affine invariance property of our algorithm at the discrete level.
Consider an arbitrary invertible affine transformation $\psi: \theta \mapsto \tilde{\theta} = T\theta + d$. In the transformed coordinate system, the parameters are
\[
\widetilde{m}_k(t_n) = T m_k(t_n) + d, \quad \widetilde{L}_k(t_n) = T L_k(t_n), \quad \widetilde{C}_k(t_n) = \widetilde{L}_k(t_n) \widetilde{L}_k(t_n)^T = T C_k(t_n) T^T.
\]
The function to be estimated in the new coordinates is defined by
\begin{align*}
\widetilde{f}_k(\theta) &= \log \rho^{\mathrm{GM}}_{\tilde{a}}\big(\widetilde{m}_k + \widetilde{L}_k\theta\big) + \widetilde{\Phi}_R\big(\widetilde{m}_k + \widetilde{L}_k\theta\big) \\
&= \log\Big(|\det T|^{-1}\rho^{\mathrm{GM}}_{a}\big(T^{-1}(\widetilde{m}_k + \widetilde{L}_k\theta - d)\big)\Big) + \Phi_R\big(T^{-1}(\widetilde{m}_k + \widetilde{L}_k\theta - d)\big) \\
&= -\log|\det T| + \log \rho^{\mathrm{GM}}_{a}\big(m_k + L_k\theta\big) + \Phi_R\big(m_k + L_k\theta\big) \\
&= f_k(\theta) - \log|\det T|.
\end{align*}
This relation immediately implies
\begin{align*}
\widetilde{E}_k(t) &= \mathbb{E}_{\mathcal{N}}\Big[\theta \theta^T\big(\widetilde{f}_k(t,\theta) - \mathbb{E}_{\mathcal{N}}[\widetilde{f}_k(t,\theta)]\big)\Big] = \mathbb{E}_{\mathcal{N}}\Big[\theta \theta^T\big(f_k(t,\theta) - \mathbb{E}_{\mathcal{N}}[f_k(t,\theta)]\big)\Big] = E_k(t).
\end{align*}
Consequently, the adaptive step size remains unchanged:
$$
\Delta \widetilde{t}_n = \min\Bigl\{ \Delta t_{\max}\eta(t_n),  \frac{\beta}{\max_k \{\|\widetilde{E}_k(t_n)\|_2\}} \Bigr\} = \Delta t_n.
$$

We now verify the affine invariance of each update. For the  covariance update, we have
\begin{align*}
\widetilde{C}_k(t_n + \Delta t_n) &= \widetilde{L}_k(t_n) e^{-\widetilde{E}_k(t_n)  \Delta {t}_n} \widetilde{L}_k(t_n)^T \\
&= (T L_k(t_n))  e^{-E_k(t_n)  \Delta {t_n}}  (T L_k(t_n))^T \\
&= T  C_k(t_n + \Delta t_n)  T^T.
\end{align*}
For the mean update, we have
\begin{align*}
\widetilde{m}_k(t_n + \Delta t_n) &= \widetilde{m}_k(t_n) - \Delta t_n \widetilde{L}_k(t_n) \mathbb{E}_{\mathcal{N}}\Big[\theta \big(\widetilde{f}_k(t_n,\theta) -\mathbb{E}_{\mathcal{N}}[\widetilde{f}_k(t_n,\theta)]\big)\Big] \\
&= (T m_k(t_n) + d) - \Delta t_n  (T L_k(t_n))  \mathbb{E}_{\mathcal{N}}\Big[\theta \big(f_k(t_n,\theta) - \mathbb{E}_{\mathcal{N}}[f_k(t_n,\theta)]\big)\Big] \\
&= T  m_k(t_n + \Delta t_n) + d.
\end{align*}
For the weight update, we have
\begin{align*}
\log \widehat{ \widetilde{w}}_k(t_n+\Delta t_n) &= \log \widetilde{w}_k(t_n) - \Delta t_n \Big(\mathbb{E}_{\mathcal{N}}[\widetilde{f}_k(t_n,\theta)] - \sum_i \widetilde{w}_i(t_n) \mathbb{E}_{\mathcal{N}}[\widetilde{f}_i(t_n,\theta)]\Big)\\
&= \log w_k(t_n) - \Delta t_n \Big(\mathbb{E}_{\mathcal{N}}[f_k(t_n,\theta)] - \sum_i w_i(t_n) \mathbb{E}_{\mathcal{N}}[f_i(t_n,\theta)]\Big)\\
&= \log \widehat{w}_k(t_n+\Delta t_n),
\end{align*}
thus $\log  \widetilde{w}_k(t_n+\Delta t_n) = \log w_k(t_n+\Delta t_n)$.
In all cases, the transformed parameters after one update step coincide with the affine transformation of the original updated parameters. This establishes the affine invariance of the complete update scheme.
\end{proof}

\section{Additional Numerical Experiments}

\subsection{Exponential convergence in Gaussian mixture case}
\label{sec:exponential}
In this subsection, we show that the exponential convergence phenomenon still occurs in noisy update iterations for the Gaussian mixture model, provided that the number and position of the Gaussian components correspond.
We use the 2-dimensional Case A from \cref{ssec:model-problems}. 
Unlike the parameter selection in \cref{ssec:model-problems}, the number of Gaussian modes selected for the used Gaussian mixture model is exactly the same as that in the target mixture.
The resulting densities and convergence rate are shown in ~\cref{fig:GMBBVI-exponential_conv}. 
The result indicates that even in the case of noisy updates, we can still observe the exponential convergence phenomenon that was proven in the noise-free situation as shown in ~\cref{sssec:Gaussian-mixture-convergence-study}.

\begin{figure}[htbp]
    \centering
    \includegraphics[scale=0.4]{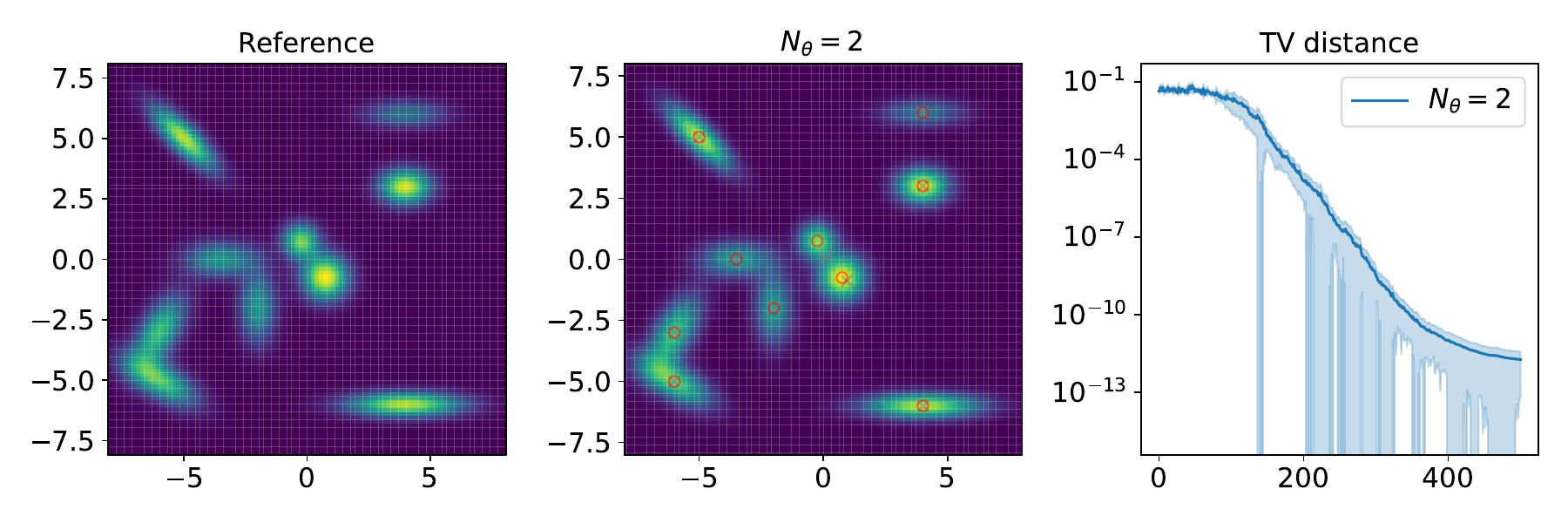}
    \caption{Exponential convergence for Gaussian mixture case. 
    The third panel reports the TV distance between the reference marginal density and the estimated marginal densities across iterations,
    where the solid line represents the mean, and shaded area represents the standard deviation 
    computed from 3 independent trials.}
    \label{fig:GMBBVI-exponential_conv}
\end{figure}

\subsection{Sensitivity study}
\label{sec:sensitivity}
In this subsection, we study the sensitivity of GMBBVI to the initial conditions, the number of modes $K$, the scheduler $\eta$, and the annealing strategy in the 10-dimensional Case A from \cref{ssec:model-problems}. For the initial conditions, we shift the initial mode means by $\pm 2$ in all components prior to annealing. For the number of modes, we consider $K = 10,20,40$. For the scheduler, we evaluate three schedulers: stable cosine decay \cref{eq:stable_cos_scheduler}, stable linear decay (replacing the cosine decay phase with a linear decay phase), and exponential decay. Finally, we study three annealing strategies: no annealing, annealing with $N_\alpha = 500$ and $\alpha=0.5$, and annealing with $N_\alpha = 500$ and $\alpha=0.1$. The resulting densities after 500 iterations are shown in \cref{fig:GMBBVI-sensitivity_study} indicating that GMBBVI is robust with respect to the initial condition, the scheduler, the annealing parameter (with annealing being essential), and the number of modes. In particular, when $K=10$, GMBBVI may miss some of the target modes, but for $K=20, 40$, all target modes are accurately recovered.

\begin{figure}[htbp]
    \centering
    \includegraphics[scale=0.25]{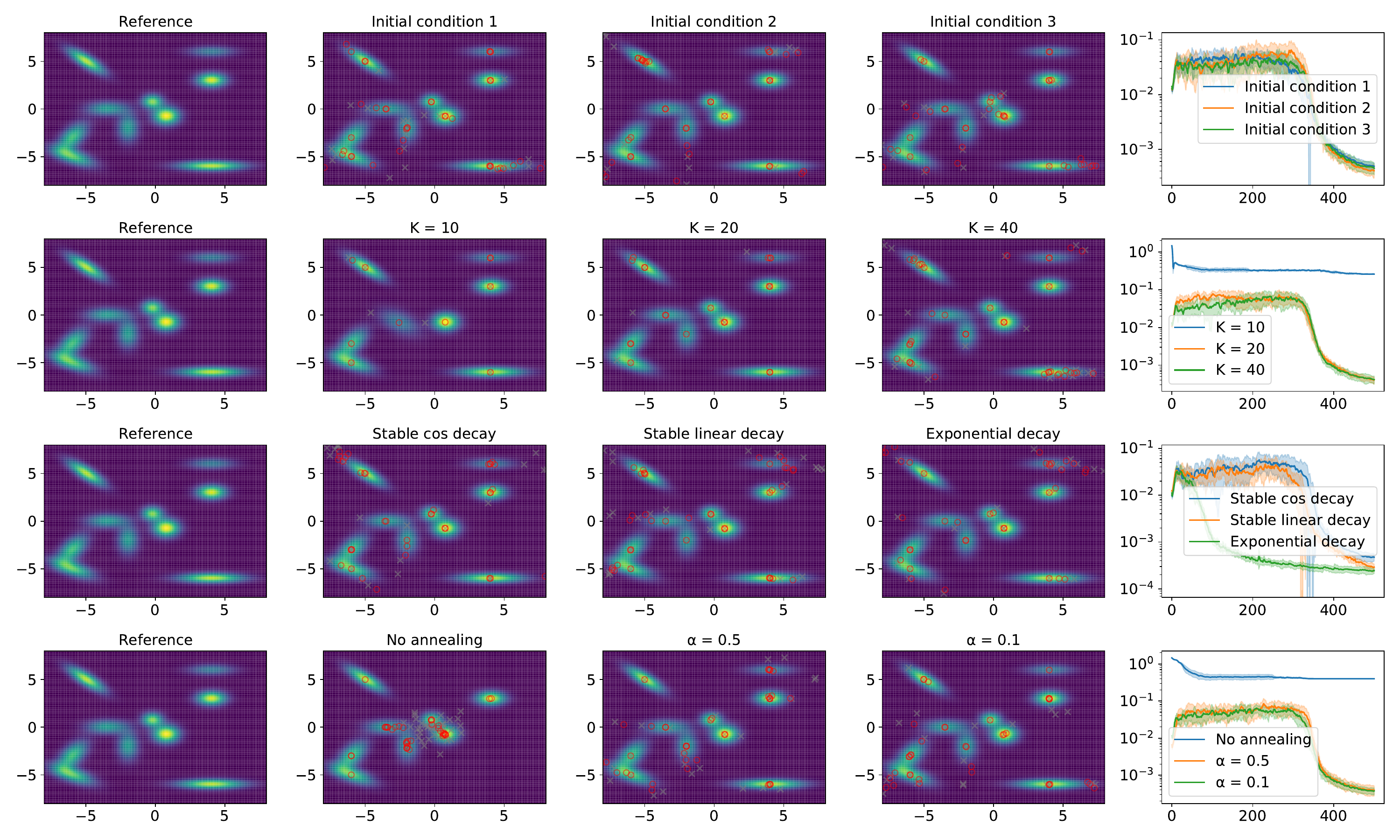}
    \caption{Sensitivity study for different choices of initial condition (first row), number of modes $K$ (second row), scheduler $\eta$ (third row) and annealing strategy (last row) in the 10-dimensional Case A.  
    The fifth panel reports the TV distance between the reference marginal density and the estimated marginal densities across iterations,
    where the solid line represents the mean, and shaded area represents the standard deviation 
    computed from 10 independent trials.}
    \label{fig:GMBBVI-sensitivity_study}
\end{figure}

\subsection{Comparison with other Gaussian mixture variational inference methods}
\label{ssec:comparison-among-VI}
In this subsection, we compare GMBBVI with two representative Gaussian mixture variational inference methods: Gaussian mixture variational inference via Wasserstein gradient flows (WGF-VI)~\cite{lambert2022variational} and Gaussian mixture Kalman inversion (GMKI)~\cite{chen2024efficient}. The comparison is conducted on the 2D Case B,  Case C from \cref{ssec:model-problems}, and Neal's funnel example. For all three methods, no annealing strategy is used and the algorithms are run for $500$ iterations. The results are reported in~\cref{fig:comparison-GMVI-methods}. We observe that WGF-VI is sensitive to the choice of step size. In our experiments,
its step sizes are manually selected as \(1.0\times 10^{-3}\), \(2.0\times 10^{-4}\),
and \(1.0\times 10^{-3}\) for the three examples, respectively. With these choices,
500 iterations are not sufficient for WGF-VI to reach the same accuracy as GMBBVI.
GMKI, on the other hand, relies on local Gaussian approximations and is therefore
less effective for targets with pronounced nonlinear geometry or high-order
curvature. These comparisons illustrate the stability and approximation advantages
of the proposed adaptive exponential scheme.
\begin{figure}
    \centering
    \includegraphics[scale=0.25]{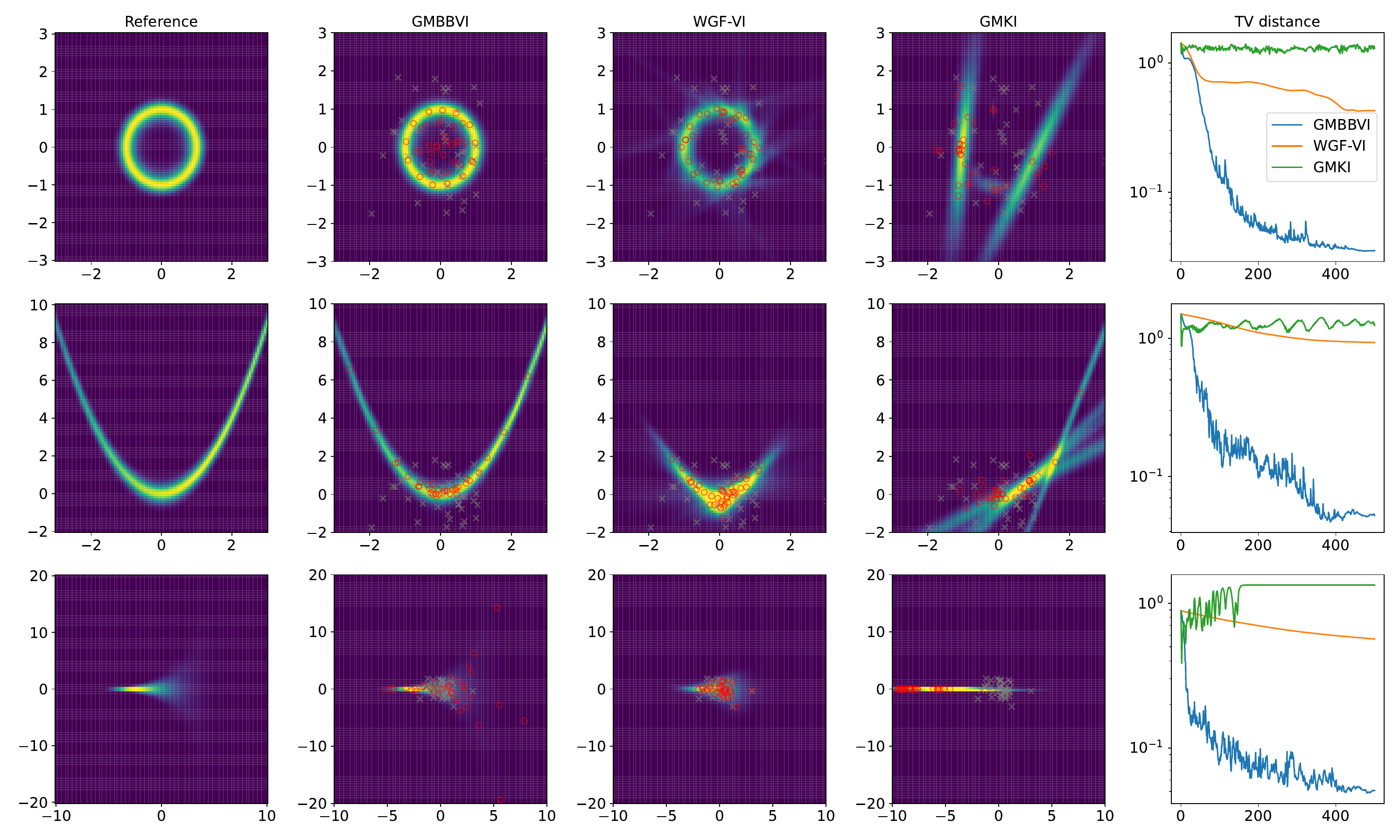}
    \caption{Comparison among GMBBVI, WGF-VI and GMKI. 2-dimensional Cases B, Case C and Neal's funnel example are arranged from the top row to the bottom row. Each panel shows the reference marginal density along with the GMBBVI, WGF-VI and GMKI estimations (from left to right). The projected means of each Gaussian component are marked by red circles, and the projected initial means are marked by grey crosses. The fifth panel displays the total variation distance between the reference marginal density and the estimated marginal densities across the iterations, with shaded area representing the standard deviation computed from 10 independent trials.}
    \label{fig:comparison-GMVI-methods}
\end{figure}

\subsection{Further discussion of adaptive step size}
\label{ssec:discuss-dt}
Since the weights are also updated in exponential form in~\cref{eq:gm-ngf-update_mw}, one may consider imposing
an additional step-size restriction based on the weight dynamics, analogous to the
restriction used for the covariance update. Specifically, the new step size is chosen by
\begin{align}
\label{eq:adaptive-dt-gaussian-weights}
    \Delta t_n' = \min\bigl\{\Delta t_{\max}\eta(t_n), \frac{\beta}{\max\|E_k(t_n)\|_2},\frac{\beta}{\max |\E_{\N}[f_k(t_n,\theta)] - \sum_i w_i \E_{\N}[f_i(t_n,\theta)]|}\bigr\}.
\end{align}
We test GMBBVI using the new step size in~\cref{eq:adaptive-dt-gaussian-weights} on the same examples and with the same initial conditions as in~\cref{ssec:model-problems}. The results are shown in~\cref{fig:dt-weights}, with all
other experimental settings and marker conventions follow \cref{fig:GMBBVI-trials}. The results indicate that this additional restriction slightly slows convergence.
We therefore do not include a weight-based control term in the adaptive step-size
rule used in the main algorithm.
\begin{figure}
    \centering
    \includegraphics[scale=0.25]{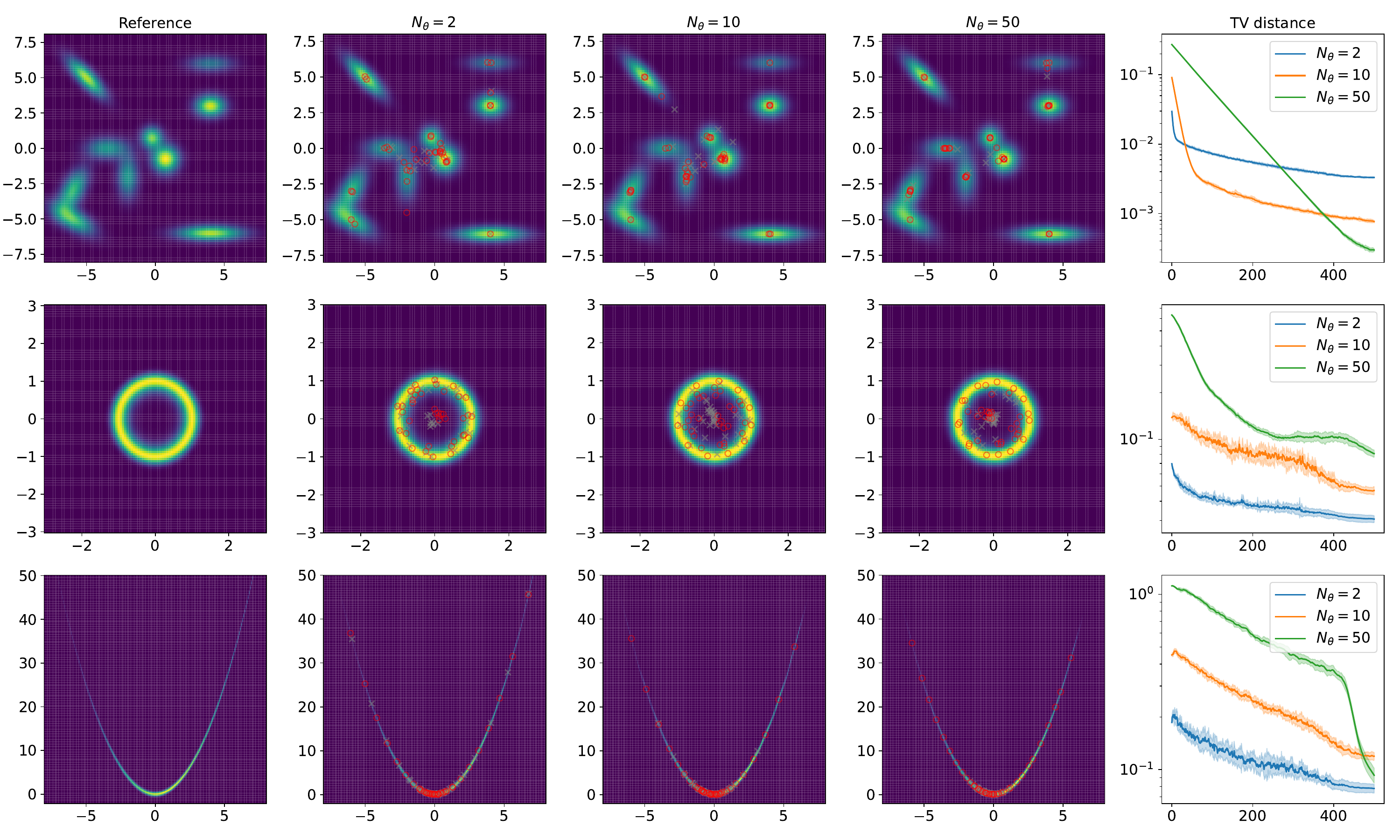}
    \caption{GMBBVI using the new step size with an additional control term based on the weights, as defined in \cref{eq:adaptive-dt-gaussian-weights}. All other experimental settings and marker conventions follow \cref{fig:GMBBVI-trials}.}
    \label{fig:dt-weights}
\end{figure}

\bibliographystyle{unsrt}
\bibliography{references}
\end{document}